\documentclass{article} 
\usepackage{configurations/iclr2026_conference,times}

\newcommand{\Pb}[1]{{\mathbb{P}}\left(#1\right) }

\providecommand{\cc}{\mathbf{c}}

\renewcommand{\ss}{\mathbf{s}}

\providecommand{\uu}{\mathbf{u}}
\providecommand{\vv}{\mathbf{v}}
\providecommand{\ww}{\mathbf{w}}
\providecommand{\xx}{\mathbf{x}}

\providecommand{\zz}{\mathbf{z}}

\providecommand{\cC}{\mathcal{C}}

\providecommand{\cH}{\mathcal{H}}
\providecommand{\cI}{\mathcal{I}}

\providecommand{\cS}{\mathcal{S}}

\providecommand{\cV}{\mathcal{V}}

\providecommand{\cZ}{\mathcal{Z}}

\providecommand{\R}{\mathbb{R}} 
\providecommand{\N}{\mathbb{N}} 

\providecommand{\mC}{\mathbf{C}}
\providecommand{\mD}{\mathbf{D}}

\providecommand{\mJ}{\mathbf{J}}

\providecommand{\mS}{\mathbf{S}}

\providecommand{\mU}{\mathbf{U}}
\providecommand{\mV}{\mathbf{V}}

\providecommand{\mX}{\mathbf{X}}

\providecommand{\mZ}{\mathbf{Z}}

\newcommand{\abs}[1]{\left\lvert#1\right\rvert}

\usepackage{hyperref}
\usepackage{url}

\usepackage{graphicx}
\usepackage{subcaption}
\usepackage[compact]{titlesec}
\renewcommand{\paragraph}[1]{\noindent\textbf{#1}}
\usepackage[page]{appendix}

\usepackage{amsmath,amssymb,amsthm}
\usepackage{thmtools,thm-restate}
\usepackage{mathtools}

\usepackage[capitalize,noabbrev]{cleveref}
\theoremstyle{plain}
\newtheorem{theorem}{Theorem}[section]

\newtheorem{lemma}[theorem]{Lemma}

\theoremstyle{definition}
\newtheorem{definition}[theorem]{Definition}
\theoremstyle{remark}

\usepackage{bm}

\usepackage{enumitem}
\usepackage{soul}
\usepackage{multirow}
\usepackage{algorithm}
\usepackage{xcolor}
\usepackage{tikz}
\usetikzlibrary{bayesnet}
\usetikzlibrary{arrows}
\usepackage{caption}
\usetikzlibrary{backgrounds}

\usepackage{wrapfig}
\usepackage{booktabs}
\usepackage{arydshln}
\usepackage{amsmath,xcolor}

\providecommand{\latel}{z}

\providecommand{\obsl}{\xx}

\providecommand{\latl}{\zz}

\providecommand{\obse}{X}
\providecommand{\disc}{D}
\providecommand{\late}{Z}
\providecommand{\sele}{S}

\providecommand{\obs}{\mX}
\providecommand{\dis}{\mD}

\providecommand{\lat}{\mZ}
\providecommand{\sel}{\mS}

\providecommand{\parents}[1]{\mathrm{Pa}(#1)}
\providecommand{\children}[1]{\mathrm{Ch}(#1)}

\providecommand{\pureparents}[1]{\mathrm{PPa}(#1)}
\providecommand{\hybridparents}[1]{\mathrm{HPa}(#1)}

\providecommand{\Lv}{L_{\mathrm{V}}}

\providecommand{\invsub}{\mV}
\providecommand{\chasub}{\mC}
\providecommand{\auxvar}{\mU}

\providecommand{\invsubl}{\vv}
\providecommand{\chasubl}{\cc}
\providecommand{\auxvarl}{\uu}

\providecommand{\invcomp}{V}

\providecommand{\auxcomp}{U}
\providecommand{\invcompl}{v}
\providecommand{\chacompl}{c}

\newcommand{\card}[1]{\left| #1 \right|}
\newcommand{\union}{\bigcup}
\newcommand{\intersection}{\bigcap}

\title{ Beyond the Black Box: Identifiable Interpretation and Control in Generative Models via Causal Minimality }

\author{
Lingjing Kong$^{*1}$, Shaoan Xie$^{*1,2}$, Guangyi Chen$^{1,2}$, Yuewen Sun$^{1,2}$, Xiangchen Song$^{1}$, \\
\textbf{~Eric P. Xing}$^{1,2}$, \textbf{Kun Zhang}$^{1,2}$  \\
$^{*}$Equal contribution \\
$^{1}$Carnegie Mellon University, Pittsburgh, PA, USA \\
$^{2}$Mohamed bin Zayed University of Artificial Intelligence, Abu Dhabi, UAE \\
}

\let\originalss\ss
\DeclareRobustCommand{\ss}{\ifmmode\mathbf{s}\else\originalss\fi}

\iclrfinalcopy 
\begin{document}

\maketitle

\begin{abstract}
    Deep generative models, while revolutionizing fields like image and text generation, largely operate as opaque ``black boxes'', hindering human understanding, control, and alignment. While methods like sparse autoencoders (SAEs) show remarkable empirical success, they often lack theoretical guarantees, risking subjective insights. 
    Our primary objective is to establish a principled foundation for interpretable generative models. 
    We demonstrate that the principle of causal minimality -- favoring the simplest causal explanation -- can endow the latent representations of modern generative models with clear causal interpretation and robust, component-wise identifiable control. We introduce a novel theoretical framework for hierarchical selection models, where higher-level concepts emerge from the constrained composition of lower-level variables, better capturing the complex dependencies in data generation. Under theoretically derived minimality conditions, we show that learned representations can be equivalent to the true latent variables of the data-generating process. Empirically, applying these constraints to leading text-to-image diffusion models allows us to extract their innate hierarchical concept graphs, offering fresh insights into their internal knowledge organization. Furthermore, these causally grounded concepts serve as levers for fine-grained model steering, paving the way for transparent, reliable systems.
\end{abstract}

\section{Introduction} \label{sec:introduction}

The transformative power of deep generative models, including diffusion models~\citep{sohldickstein2015deep,ho2020denoising,rombach2021highresolution,song2022denoising,dhariwal2021diffusion,nichol2021improved} and language models~\citep{radford2018improving,radford2019language,brown2020language,raffel2020exploring}, is reshaping numerous domains. However, their escalating complexity and scale frequently cast them as opaque ``black boxes''~\citep{shwartz2017opening,olah2020zoom}. This opacity presents a formidable barrier to genuine human understanding, severely curtails our ability to exert precise control over their behavior~\citep{jahanian2020steerability, harkonen2020ganspace, shen2020interpreting, wu2020stylespace}, and complicates the crucial alignment with human values and intentions.

Although recent empirical tools, such as sparse autoencoders (SAEs) for large language models (LLMs)~\citep{cunningham2023sparse, huben2023sparse, gao2024scaling} and diffusion models~\citep{surkov2024unpacking, kim2024textit, kim2025concept, cywinski2025saeuron, huang2025tide}, offer avenues for probing these models, a fundamental gap persists. Without rigorous theoretical underpinnings, interpretations derived from these methods risk being subjective or susceptible to human biases, rendering them potentially untrustworthy for risk-sensitive applications~\citep{kaddour2022causal, moran2025towards,scholkopf2021toward}. 
In this work, we directly tackle this critical challenge, seeking to establish a principled foundation for interpretable and controllable generative models. 

Our investigation centers on two questions: Under what \emph{theoretical conditions} can we reliably identify meaningful, interpretable latent concepts within the intricate architectures of modern generative models? And, crucially, what \emph{actionable, theoretically-grounded insights} can empower us to advance both the interpretability and the controllability of these powerful systems?

Towards these goals, we identify the \emph{causal minimality}~\citep{peters2017elements, spirtes2000causation, hitchcock1997probabilistic} principle as the formal underpinning that connects widespread practices, such as enforcing sparsity, to the recovery of meaningful, interpretable concepts. 
This principle, advocating for the simplest causal model consistent with observations, allows for the identification of latent hierarchical concept structures. 
In our context, minimality translates to sparsity in the concept graphs. 
We explore its application to text-to-image (T2I) diffusion models~\citep{ramesh2021zero,ramesh2022hierarchical, rombach2021highresolution}. Our findings indicate that imposing sparsity constraints on internal representations is instrumental for identifying intrinsic visual concepts.

A cornerstone of our contribution is establishing the first identifiability results for \emph{selection}-based~\citep{zheng2024detecting,spirtes1995causal,hernan2004structural,zhang2008completeness,bareinboim2022recovering,forre2020causal,correa2019identification,chen2024modeling} hierarchical models. In such models, higher-level variables emerge as effects of compositions of lower-level variables, where higher-level variables control and select the configuration of lower-level ones. This fundamentally diverges from traditional hierarchical causal models~\citep{pearl2009causality,choi2011learning,zhang2004hierarchical}, in which causal influence typically propagates from higher to lower levels. 
The selection model structure is particularly adept at capturing the intricate conditional dependencies among low-level features for forming coherent high-level concepts -- it explains how specific arrangements of wheels, doors, and a roof constitute a recognizable ``car'', rather than a disjointed collection of parts. Traditional hierarchical models often neglect such intra-level dependencies by assuming no within-layer causal edges, as explicitly modeling them would yield overly dense graphs. 
The selection mechanism, in contrast, offers a simpler approach to this essential coordination. 
Its adherence to the minimality principle strongly favors it as a more accurate representation of the true model.

Despite the appeal, their identifiability has been underexplored. Prior research has largely centered on traditional hierarchical structures. Moreover, their techniques often rest on simplifying assumptions (e.g., linearity~\citep{xie2022identification, huang2022latent, dong2023versatile, anandkumar2013learning} or achieve only subspace-level identifiability~\citep{kong2023identification}). Such methods are generally inapplicable to the hierarchical selection models. Our framework is the first to establish \emph{component-wise} identifiability for \emph{continuous hierarchical selection} models. Specifically, we demonstrate that under the minimality conditions in Conditions~\ref{cond:vision_identification}-\ref{asmp:sparsity}, the learned representations are equivalent to the true latent variables of the underlying hierarchical process. This disentanglement of individual, atomic concepts is what affords significantly more nuanced interpretability and precise control in the resulting generative models.

By applying the derived sparsity constraints to state-of-the-art generative models, we successfully extract their innate hierarchical concept graphs (Figure~\ref{fig:topic_figure}). This not only illuminates their internal knowledge organization but also shows that causally-grounded concepts serve as highly effective levers for model steering. Our experiments illustrate key implications of our theorems and show how a principled, causal understanding can guide the application of established interpretation techniques.


\begin{figure}[!t]
\centering
\includegraphics[width=0.9\textwidth]{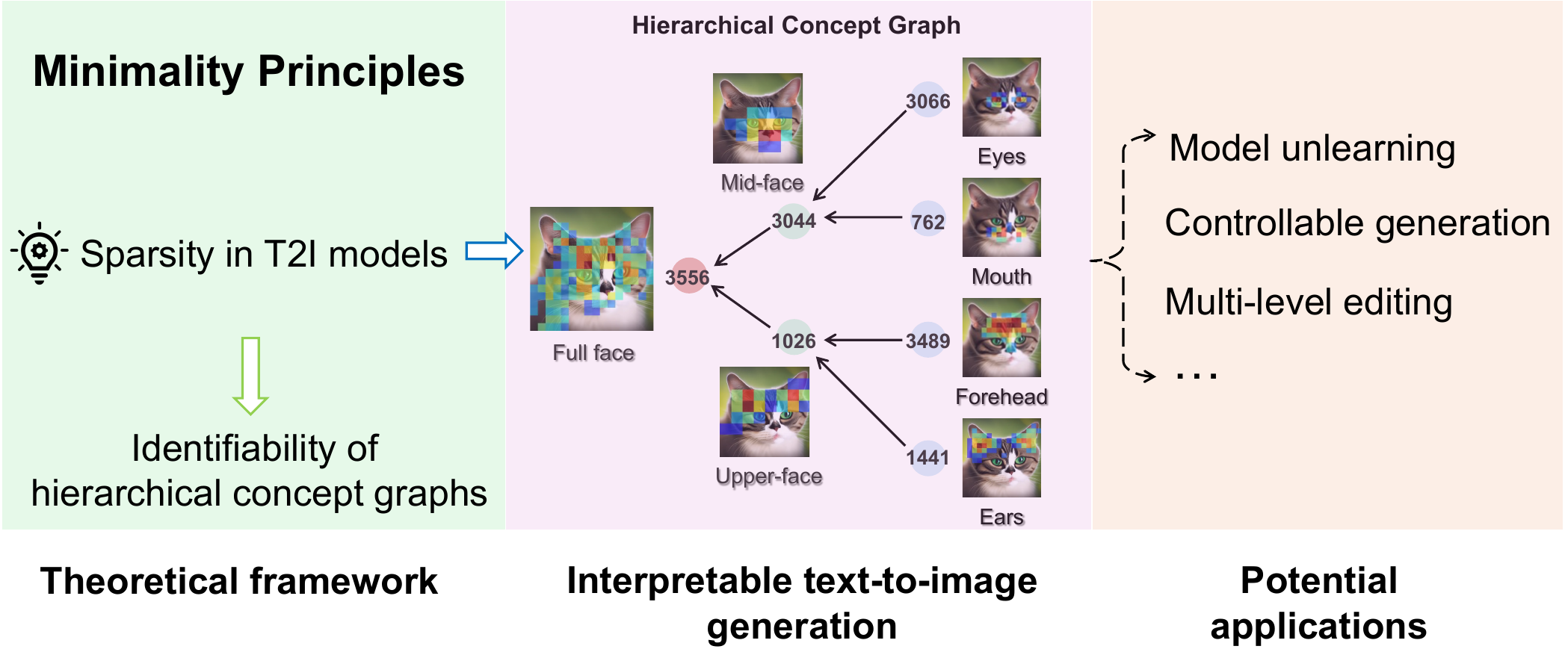}
\caption{
    \small
   Our causal minimality principle enables interpretable text-to-image generation through hierarchical concept graphs, with implications for downstream tasks.
    }
    \vspace{-0.7cm}
\label{fig:topic_figure}
\end{figure}

\section{Related Work} \label{sec:related_work}
\vspace{-0.1cm}
\paragraph{Hierarchical models.}
Complex real-world data distributions frequently exhibit inherent hierarchical structures among their underlying latent variables, a characteristic that has motivated extensive research. Initial explorations primarily focus on continuous latent variables with linear interactions~\citep{xie2022identification,huang2022latent,dong2023versatile,anandkumar2013learning}. Other lines of work have centered on discrete latent variables; however, these approaches are often constrained in their applicability to continuous data modalities like images~\citep{Pearl88, zhang2004hierarchical, choi2011learning, gu2023bayesian,kong2024learning}. Furthermore, prevalent latent tree models, which connect variables via a single undirected path~\citep{Pearl88,zhang2004hierarchical,choi2011learning}, risk oversimplifying the multifaceted relationships present in complex systems. More recently, while \citet{park2024geometry} make progress in capturing geometric properties of language model representations using hierarchical models, their work does not address the critical issue of latent variable identification. \citet{kong2023identification} tackle nonlinear, continuous latent hierarchical models, but their framework, operating under rather opaque functional conditions, falls short of component-wise identifiability, thereby leaving room for concept entanglement.
Our work distinctively investigates \emph{selection} hierarchical models, contending that their structural properties yield a more faithful representation of latent concepts in natural data distributions. 
In these models, latent variables function as colliders, a significant departure from their role as confounders in the aforementioned prior art. This critical distinction renders existing identification techniques largely inapplicable. 
To the best of our knowledge, we are the first to provide \emph{component-wise} identifiability for \emph{both continuous and discrete hierarchical selection models}.

\paragraph{Interpretability for generative models.}
Despite the remarkable advancements of generative models, their internal mechanisms often remain opaque. This presents a significant challenge to understanding and control. Considerable research has focused on obtaining interpretable features to enable more controllable generation. 
Early efforts center on analyzing the latent space of generative adversarial networks, e.g., ~\citep{harkonen2020ganspace,voynov2020unsupervised,shen2020interfacegan}. 
Recently, sparse autoencoders (SAEs) have gained prominence for interpreting hidden representations, particularly in language models. These studies show that SAEs trained on transformer residual-stream activations can identify latent units corresponding to linguistically meaningful features \citep{cunningham2023sparse, huben2023sparse,gao2024scaling,mudide2025efficient,shi2025routesae}. 
These interpretability techniques have also been successfully extended to diffusion models. \citet{surkov2024unpacking} reveal interpretable features and specialization across diffusion model blocks. Other work trains SAEs with lightweight classifiers on diffusion model features~\citep{kim2024textit} or steers generation away from undesirable visual attributes~\citep{huang2025tide}.
Our hierarchical approach is related to recent findings on the evolution of semantics during the diffusion process. It has been observed that high-level concepts, such as object shape and structure, tend to emerge in earlier, high-noise timesteps, while fine-grained, low-level details are synthesized in later, low-noise stages~\citep{patashnik2023localizing,tinaz2025emergence,mahajan2024prompting}. While these works provide valuable empirical validation of this phenomenon, our work offers a new perspective by framing these observations within a formal hierarchical, causal structure. We provide a theoretical foundation, rooted in causal minimality and selection models, to explain \emph{how} these concepts compose and, crucially, \emph{under what conditions} they can be provably identified.
Our approach also relates to generative concept bottleneck models, which achieve interpretability by forcing predictions through a bottleneck layer of concepts~\citep{ismail2024concept,kulkarni2025interpretable}. While these methods provide powerful intervention capabilities by design, our work differs by focusing on the discovery of the innate hierarchical and causal concept structure in the data. We provide the theoretical conditions for identifying these concepts component-wise, allowing us to then use this discovered graph for fine-grained multi-level interventions.

\paragraph{Decomposition-based interpretability.} Our work is fundamentally distinct from post-hoc, decomposition-based interpretability methods, such as the prototype-matching approach~\citep{chen2019looks}. This line of research, while pioneering, has known limitations (often stemming from its prototype-based implementation): its reliance on class-label supervision can lead to non-compositional, class-locked concepts~\citep{Rymarczyk2021ProtoPShare}, and its use of rigid patch-matching struggles with context and deformation~\citep{donnelly2022deformable,xue2024protopformer}.
In contrast, our approach is \emph{class/object agnostic} (similar to SAEs) and \emph{context-sensitive}, learning from the raw generative data without class labels. Our approach learns compositional, shared concepts (e.g., a single ``furry texture'' from ``cats'' and ``pandas'') rather than rigid, class-specific prototypes. This enables the causal, interventional control (e.g., Figure~\ref{fig:shared_concepts} and downstream tasks in Section~\ref{subsec:downstream_exp}) that prototype-matching cannot guarantee.
Please find additional related work in Appendix~\ref{app:related_work}.
\vspace{-0.1cm}

\vspace{-0.2cm}
\section{Deep Generative Models as Hierarchical Concept Models} \label{sec:formulation}

\paragraph{Notations.}
We denote random variables with upper-case characters (e.g., $X$) and values with lower-case characters (e.g., $x$).
We distinguish multidimensional objects with bold fonts (e.g., $\mX$) and refer to their dimensionality as $ n(\cdot) $.
We view multidimensional variables as \emph{sets} when appropriate (e.g., $\obs$ as $\{ \obse_{i} \}_{i \in [n(\obs)]}$). 
Parents $\parents{\cdot}$ and children $\children{\cdot}$ relations are defined based on the selection graph (Figure~\ref{fig:vision_causal_graph}).
If $ X $ has only one child $ Y $, we refer to $ X $ as a pure parent of $ Y $, i.e., $ X \in \pureparents{Y} $; if $X$ has other children than $ Y$, we refer to $ X$ as a hybrid parent of $Y$, i.e., $ X \in \hybridparents{Y} $.
We denote the set of natural numbers $ \{1, \dots, M \} $ as $ [M] $. More background information is in Appendix~\ref{app:dicussions}.

We denote the image as the continuous variable $\obs \in \R^{n(\obs)}$ and text as the discrete variable $\dis \in \N^{n(\dis)}$. Visual concepts are $\lat := [ \lat_{1}, \cdots, \lat_{\Lv} ]$, where $\Lv$ is the number of visual hierarchical levels and $\lat_{l} \in \R^{n(\lat_{l})}$ are concepts at level $l$ (Figure~\ref{fig:vision_causal_graph}).
The discrete variables $\mathbf{D}$ capture the discrete nature of textual concepts (like ``cat'' or ``bicycle''). In contrast, the visual concepts ($\mathbf{Z}$) are continuous to represent rich visual details. $\mathbf{D}$ acts as a selection variable that governs the joint configuration of the continuous $\mathbf{Z}$ variables. For instance, the discrete concept ``bicycle'' ($\mathbf{D}$) selects for a coherent arrangement of continuous visual features ($\mathbf{Z}$) representing wheels, a frame, and handlebars, rather than a random collection of those continuous parts.

\begin{wrapfigure}{r}{0.45\textwidth}
    \centering
    \begin{tikzpicture}[scale=0.45, line width=0.6pt, inner sep=0.6mm, shorten >=.1pt, shorten <=.1pt]
        \tikzset{
            znode/.style={text=brown},
            dnode/.style={text=blue},
            xnode/.style={text=black},
            every node/.style={align=center},
            edge from parent/.style={draw,->}
        }

        \node[dnode] (d1) at (-1.5,1) {$D_{1}$};
        \node[dnode] (d2) at (1.5,1) {$D_{2}$};

        \node[znode] (z11) at (-1.5,-1) {$Z_{1, 1}$};
        \node[znode] (z12) at (1.5,-1) {$Z_{1, 2}$};

        \node[znode] (z21) at (-3,-3) {$Z_{2,1}$};
        \node[znode] (z22) at (-1,-3) {$Z_{2,2}$};
        \node[znode] (z23) at (1,-3) {$Z_{2,3}$};
        \node[znode] (z24) at (3,-3) {$Z_{2,4}$};

        \node[znode] (z31) at (-3.75,-5) {$Z_{3,1}$};
        \node[znode] (z32) at (-2.25,-5) {$Z_{3,2}$};
        \node[znode] (z33) at (-0.75,-5) {$Z_{3,3}$};
        \node[znode] (z34) at (0.75,-5) {$Z_{3,4}$};
        \node[znode] (z35) at (2.25,-5) {$Z_{3,5}$};
        \node[znode] (z36) at (3.75,-5) {$Z_{3,6}$};

        \draw[-latex] (z11) -- (d1);
        \draw[-latex] (z11) -- (d2);
        \draw[-latex] (z12) -- (d2);

        \draw[-latex] (z21) -- (z11);
        \draw[-latex] (z22) -- (z11);
        \draw[-latex] (z23) -- (z11);

        \draw[-latex] (z23) -- (z12);
        \draw[-latex] (z24) -- (z12);
        \draw[-latex] (z22) -- (z12);

        \draw[-latex] (z31) -- (z21);
        \draw[-latex] (z31) -- (z22);

        \draw[-latex] (z32) -- (z21);
        \draw[-latex] (z32) -- (z22);

        \draw[-latex] (z33) -- (z22);

        \draw[-latex] (z34) -- (z23);

        \draw[-latex] (z35) -- (z22);
        \draw[-latex] (z35) -- (z24);

        \draw[-latex] (z36) -- (z23);
        \draw[-latex] (z36) -- (z24);

        \node[xnode, scale=1.2] (x) at (0,-7) {$\obs$};

        \draw[latex-, line width=1pt] (0,-5.25) -- (0,-5.75) -- (x);

        \node[dnode, align=center, font={\footnotesize}] at (-6,1) {Text};
        \node[znode, align=center, font={\footnotesize}] at (-6,-3) {Visual \\ Concepts};
        \node[xnode, align=center, font={\footnotesize}] at (-6,-7) {Image};
    \end{tikzpicture}
    \caption{
        \small
        \textbf{A visual concept graph.} We denote text as $\dis$, visual concepts as $\lat$, and the image as $\obs$. High-level concepts function as selection variables for low-level variables.
    }
    \label{fig:vision_causal_graph}
\end{wrapfigure}
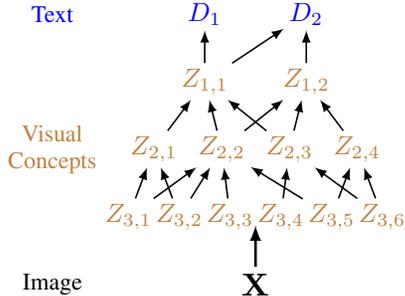

\paragraph{Hierarchical processes and selection mechanisms.}
Our framework conceptualizes high-level concepts as emerging from or being effects of lower-level concepts. 
This is captured by a \emph{selection mechanism}~\citep{zheng2024detecting,spirtes1995causal,hernan2004structural,zhang2008completeness,bareinboim2022recovering,forre2020causal,correa2019identification,chen2024modeling}, where variables $\mV_{l}$ at a higher level of abstraction (smaller $l$) is determined by its constituent, more detailed components $\mV_{l+1}$ (i.e., its ``parents'').
The selection function $g_{\mV_l}$ maps these lower-level constituents to the higher-level concept:
\begin{align} \label{eq:selection_mechanism}
\mV_{l} := g_{\mV_l}(\mV_{l+1}).
\end{align}
In other words, $\mV_l$ is a selection variable over $\mV_{l+1}$. 
In many natural data distributions of interest, we can only observe the data points for which the selection criterion is met, i.e., $ \mV_{l} $ only takes on a strict subset of its range $ \Omega $.
Therefore, the distribution of $ \mV_{l+1} $ is always the conditional distribution $\Pb{\mV_{l+1} | \mV_{l} }$.
This conditioning on $\mV_l$ can induce dependencies among components in $\mV_{l+1}$. 
For instance, if $V_{l+1,i} \to V_l \leftarrow V_{l+1,j}$, conditioning on $V_l$ makes $V_{l+1,i}$ and $V_{l+1,j}$ dependent. 

Under this formulation, one can leverage the inverse process of \eqref{eq:selection_mechanism} to sample observable data (images, text), proceeding from higher-level abstract concepts to lower-level concrete details:

\begin{align} \label{eq:vision_data_generating_process}
          \lat_{0} \sim \Pb{ \lat_{0} }, \quad \lat_{l} \sim \Pb{ \lat_{l} | \lat_{l-1} }, \quad l \in \{1, \dots, \Lv + 1\},
\end{align}
where we denote $ \lat_{0}:= \dis $ and $ \lat_{\Lv+1}:=\obs$.
While \eqref{eq:vision_data_generating_process} defines the generative pathway, the underlying structure is shaped by the selection principle of \eqref{eq:selection_mechanism}: the conditional distributions in \eqref{eq:vision_data_generating_process} are implicitly learned if one has learned selection mechanisms in \eqref{eq:selection_mechanism} and vice versa.

\paragraph{Why is this ``selection'' formulation?}
The ``selection'' perspective is critical for modeling how abstract concepts enforce coherence among their more concrete constituents. 
Consider generating an image of a ``bicycle'' (a high-level concept $\lat_l$). Its components -- wheels, frame, handlebars (lower-level concepts $\lat_{l+1}$) -- must not only be present but also be arranged in a specific, structurally sound configuration. 
Traditional hierarchical models~\citep{choi2011learning,Pearl88,zhang2004hierarchical} assume independent low-level concepts $ \late_{l+1, i} $ given high-level concepts $ \lat_{l} $ and stochastically sample these components, which could lead to unrealistic arrangements (e.g., wheels detached from the frame if the learned conditional is not perfect).
Therefore, these models must additionally incorporate causal edges within each hierarchical level to capture this conditional dependency, resulting in highly dense causal graphs. 
In contrast, the selection model, by positing that $\lat_l$ is an effect of a specific configuration of $\lat_{l+1}$, emphasizes that the ``bicycle'' concept arises from a coherent selection and composition of its parts. 
This structured dependency, induced by the selection mechanism, yields a much simpler graphical model to describe the natural data distribution, thus preferred by the \emph{minimality} principle.

\paragraph{Connections to text-to-image diffusion models.}
The iterative denoising process in diffusion aligns with our hierarchical data construction. These models involve a sequence of transformations $\{f_t\}_{t=1}^T$, parameterized by timestep $t$, that progressively restore a less noisy image $\obs_t$ from a more corrupted version $\obs_{t+1}$. 
As interpreted by \citet{kong2024learning}, each $f_{t+1}$ can be viewed as an autoencoder: it extracts a representation $\lat_{\cS(t+1)}$ ($\cS(t+1)$ indexes U-Net features associated with timestep $t+1$) from the noisy input $\obs_{t+1}$, and uses this representation to produce the less noisy $\obs_t$. 
In this view, representations $\lat_{\cS(t+1)}$ from higher noise levels (larger $t$, where $\obs_{t+1}$ is closer to pure noise) correspond to higher-level, more abstract concepts in our hierarchy (e.g., $\lat_l$ with smaller $l$), as fine-grained details are obscured by noise. 
Conversely, representations from lower noise levels (smaller $t$) capture more concrete details (e.g., $\lat_l$ with larger $l$).
The diffusion model's step-wise refinement thus mirrors our hierarchical generation $\Pb{\lat_{l+1} | \lat_l}$, with the initial text prompt $\dis$ typically guiding the most abstract visual concepts (e.g., $\lat_1 \sim \Pb{\lat_1 | \dis}$, Figure~\ref{fig:topic_figure}).
In our empirical analysis (Section~\ref{sec:vision_exp}), we explicitly map distinct diffusion timesteps to these hierarchical levels: high noise levels (e.g., $t=899$) correspond to abstract concepts, and low noise levels (e.g., $t=100$) to fine-grained details.

\paragraph{Identifiability and interpretability.}
In light of the connection, a crucial question remains: are the internal representations learned by these models (e.g., U-Net features, transformer activations) truly reflective of the ground-truth concepts of the data, or are they merely effective for the generation task without being inherently interpretable and controllable? 
This motivates the need for \emph{identifiability} guarantees that affirm the equivalence between the two worlds, which we present in Section~\ref{sec:theory}.

\section{Identifiable Representations under Causal Minimality} \label{sec:theory}

We first formally define our core theoretical principle, \emph{causal minimality}~\citep{peters2017elements,spirtes2000causation,hitchcock1997probabilistic}: Among all causal models that can explain the observed data, the true model is the simplest one. 
This principle is the key to our goal of identifiability (Definition~\ref{def:componentwise}). Causal minimality, as a principle, manifests as concrete, enforceable mechanisms in specific settings. In the setting we study, this mechanism is sparse connectivity in the causal graph (our minimality condition, \ref{cond:vision_identification}-\ref{asmp:sparsity}). Enforcing this sparsity is thus the practical mechanism that provides theoretical guarantees for identifiability.

For visual concepts, minimality manifests as a preference for \emph{sparse graphical dependencies} within the latent hierarchy. 
This implies that concepts are formed through a limited set of direct causal influences, making the underlying structure easier to discern.

A key challenge we address is the identifiability of \emph{hierarchical selection models}. In these models, higher-level concepts are effects of lower-level concepts. This contrasts with traditional hierarchical models where causality often flows from abstract to concrete, and where latent variables typically act as confounders~\citep{Pearl88,zhang2004hierarchical,choi2011learning,gu2023bayesian,kong2024learning,xie2022identification,huang2022latent,dong2023versatile,kong2023identification,anandkumar2013learning}. In our selection framework, latent variables act as colliders, rendering many existing identifiability results inapplicable. This distinction necessitates the novel theoretical development presented herein.
Our goal is to achieve \emph{component-wise identifiability}:
\begin{definition}[Component-wise Identifiability] \label{def:componentwise}
Let $\lat$ and $\hat{\lat}$ be variables under two model specifications. We say that $\lat$ and $\hat{\lat}$ are \emph{identified component-wise} if there exists a permutation $\pi$ such that for each $i \in [n(\lat)]$, $ \hat{\late}_{i} = h_{i}(\late_{\pi(i)})$ where $ h_{i} $ is an invertible function.
\end{definition}
This strong form of identifiability ensures that each learned latent component $\hat{\late}_{i}$ corresponds to a single true latent component $\late_{\pi(i)}$. This is vital for unambiguous interpretation and targeted control. We assume the standard faithfulness condition~\citep{spirtes2001causation}, meaning the graphical model accurately reflects all conditional independence relations in the data.


In the following, we consider the identification of continuous latent visual concepts $\lat$. 

\newcommand{\wvectorcontent}{%
\ww( \tilde{\latl}, \latel ) 
    = \Big(
        \frac{\partial \log p \left(\tilde{\latl} | \latel \right)}{\partial \tilde{\latel}_{1} }, \ldots, \frac{\partial \log p \left(\tilde{\latl} | \latel \right)}{\partial \tilde{\latel}_{n( \tilde{\latl} ) } }
    \Big)
}

\begin{restatable}[Visual Concept Identification Conditions]{condition}{visionidentificationconditions} \label{cond:vision_identification} {\ }
    \begin{enumerate}[label=\roman*,leftmargin=2em, topsep=0.5pt, partopsep=0pt, itemsep=-0.0em]
        \item \label{asmp:invertibility} \textbf{Informativeness}: There exists a diffeomorphism $g_{l}: \left( \lat_{l}, \bm{\epsilon}_{l} \right) \mapsto \obs$ for $ l \in [0, L] $, where $\bm{\epsilon}_{l} $ denotes independent exogenous variables.
        \item \label{asmp:smooth_density} \textbf{Smooth Density}: The probability density function $ p( \latl_{l+1} | \latl_{l} ) $ is smooth for any $l \in [\Lv]$.
        \item \label{asmp:linear_independence} \textbf{Sufficient Variability}: For each $\late$ and its parents $\tilde{\lat}:= \parents{\late}$, at any value $\tilde{\latl}$ of $\tilde{\lat}$, there exist $n(\tilde{\lat})+1$ distinct values of $\late$, denoted as $\{\latel^{(n)}\}_{n=0}^{n(\tilde{\late})}$, such that the vectors $\ww(\tilde{\latl}, \latel^{(n)})-\ww(\tilde{\latl}, \latel^{(0)})$ are linearly independent where $\wvectorcontent$.
        
        \item \label{asmp:sparsity} \textbf{Sparse Connectivity (Minimality)}: For each parent concept $\tilde{\late}$, there exists a subset of its children $\lat \subseteq \children{\tilde{\late}}$ such that their \emph{only} common parent is $\tilde{\late}$, i.e., $ \bigcap_{ \late \in \lat } \parents{ \late } = \{ \tilde{\late} \} $. 
    \end{enumerate}
\end{restatable}

\paragraph{Interpreting Condition~\ref{cond:vision_identification}.}
Condition~\ref{cond:vision_identification}-\ref{asmp:invertibility} ensures that the observed data $\obs$ (e.g., an image) fully captures the information about the latent concepts $\lat_l$. This is a natural assumption as high-dimensional observations contain rich information. Condition~\ref{cond:vision_identification}-\ref{asmp:smooth_density} is a standard regularity assumption for analysis.
Both are common in nonlinear ICA literature~\citep{hyvarinen2016unsupervised,hyvarinen2019nonlinear,khemakhem2020icebeem,khemakhem2020variational,von2021self,kong2023identification}.
Condition~\ref{cond:vision_identification}-\ref{asmp:linear_independence} formalizes the idea that distinct lower-level concepts (e.g., ``wheel,'' ``door'') respond in sufficiently distinct ways to changes in a shared higher-level concept (e.g., ``car''), thus facilitating the identification of these lower-level concepts.
Condition~\ref{cond:vision_identification}-\ref{asmp:sparsity} is an instantiation of causal minimality for visual concepts. It posits that the causal graph of concepts is sparse -- each concept has a unique ``fingerprint'' in terms of its connectivities. 
This sparsity is crucial for disentanglement~\citep{zheng2022identifiability,lachapelle2024nonparametric,xu2024sparsity,lachapelle2022disentanglement,lachapelle2022synergies} and is a less restrictive assumption than, for example, pure observed children for each latent variable~\citep{arora2012learning,arora2012practical,moran2021identifiable}. 
This condition formalizes a core principle: concepts are learned through \emph{comparison}. A concept is identifiable only if the data is rich enough to distinguish it from alternatives. For instance, if ``Knight'' and ``Horse'' always co-occur, they are learned as a fused concept; learning them separately requires data that breaks this correlation.

\begin{restatable}[Visual Concept Identification]{theorem}{visionidentification} \label{thm:vision_identification} {\ } 
Assume the process for visual concepts in \eqref{eq:vision_data_generating_process}. If a model specification $\bm\theta_{\mathrm{V}}$ satisfies Condition~\ref{cond:vision_identification}, and an alternative specification $\hat{\bm\theta}_{\mathrm{V}}$ satisfies Conditions~\ref{cond:vision_identification}-\ref{asmp:invertibility} and \ref{cond:vision_identification}-\ref{asmp:smooth_density}, along with a sparsity constraint such that for corresponding $\hat{\late}$ and $\late$:
\begin{align} \label{eq:sparsity_constraint}
    n( \parents{ \hat{\late} } ) \leq n( \parents{ \late } ),
\end{align}
then, if both models $\bm\theta_{\mathrm{V}}$ and $\hat{\bm\theta}_{\mathrm{V}}$ generate the same observed data distribution $\Pb{\obs}$, the latent visual concepts $\lat_{l}$ are component-wise identifiable for every level $l \in [\Lv]$.
\end{restatable}

\paragraph{Proof sketch for Theorem~\ref{thm:vision_identification}.}
The proof proceeds by identifying the hierarchical model level by level, from the top (most abstract concepts) $\lat_1$ downwards to $\lat_{\Lv}$.
1) The paired text data $\dis$ acts as an auxiliary variable, providing diverse ``influences'' on the top-level $\lat_1$. Condition~\ref{cond:vision_identification}-\ref{asmp:linear_independence} ensures these interventions have distinguishable effects. Analogous to techniques in nonlinear ICA~\citep{hyvarinen2016unsupervised,hyvarinen2019nonlinear,kong2022partial}, each component $\disc$ allows the identification of the subspace of $\lat_1$ variables it influences.
2) With these subspaces identified, one can identify the intersection of these subspaces~\citep{von2021self,yao2023multi,Kong2023understanding}.
Therefore, if the graphical structure is sufficiently sparse, as specified in Condition~\ref{cond:vision_identification}-\ref{asmp:sparsity}, one can identify the top-level latent variable $ \lat_{1} $ component-wise.
3) Once $\lat_1$ is identified, its components can serve as the auxiliary variables to identify the next level, $\lat_2$. This process is repeated iteratively down the hierarchy, identifying $\lat_l$ using the already identified $\lat_{l-1}$.

\paragraph{Implications for text-to-image diffusion models.}
Theorem~\ref{thm:vision_identification} underscores that the \emph{sparsity constraint \eqref{eq:sparsity_constraint} is pivotal for identifying true visual concepts}. 
In practice, this constraint is instantiated through a two-step process: 1) Level-specific concept learning: We train $K$-sparse SAEs on features at the specific timesteps defined in Section~\ref{sec:formulation}. This approximates the sparsity condition required by Theorem~\ref{thm:vision_identification}. 2) Cross-level causal discovery: We then apply causal discovery algorithms (e.g., PC~\citep{spirtes2001causation}) across these sparse features to construct the hierarchical graph, validating that the learned representations align with the theoretical identification guarantees.

\section{Experiments} 
\label{sec:vision_exp}

\paragraph{Evaluation design and objectives.} We design our experiments to validate our theoretical framework in two ways. In Section~\ref{subsec:interpretation_exp}, we provide a direct empirical test of our theory: we apply the sparsity constraints derived from causal minimality (Condition~\ref{cond:vision_identification}) and show that we can, as predicted, extract a meaningful and interpretable hierarchical concept graph. In Section~\ref{subsec:downstream_exp}, we demonstrate the utility of these identified concepts. If our concepts are truly component-wise identifiable (Definition~\ref{def:componentwise}), they should be individually controllable. We test this via a suite of challenging downstream tasks—including model unlearning, controllable image generation, and multi-level editing. For example, we compare against state-of-the-art unlearning methods to rigorously benchmark our concept removal capabilities. Our objective is to show that our theory not only finds interpretable concepts but also provides a practical mechanism for fine-grained, reliable model control. More detailed settings for each experiment are provided in their respective subsections.


\paragraph{Hierarchical causal analysis.}
Our theoretical framework motivates an empirical analysis that differs from standard interpretability approaches. 
Following the framework established in Sections~\ref{sec:formulation} and \ref{sec:theory}, we apply our two-step identification process to Stable Diffusion (SD) 1.4~\citep{rombach2022high} and Flux.1-Schnell~\citep{flux2024} (Appendix~\ref{app:empirical_results}). We analyze feature representations at the previously defined timesteps (899, 500, and 100) to extract and verify the hierarchical concept graph.

\paragraph{Benefits.}
This hierarchical perspective provides two main benefits. First, it enables compositional editing. For a complex object like ``a textured tree stump'', our analysis can distinguish the ''stump'' (a mid-level concept) from its ``texture'' (a low-level one), allowing for independent steering. This is a fine-grained control challenging for non-hierarchical methods that tend to learn entangled features (see Table~\ref{fig:multi-level}). Second, it allows for targeted intervention. By identifying a concept's level, we can inject a steered feature back into the diffusion process only at its corresponding timestep, which helps in reducing the unwanted artifacts that can arise from applying steering globally across all timesteps (see Figure~\ref{fig:concept_removal}). More details in Appendix~\ref{app:implementation_details} and Figure~\ref{fig:method_diagram}.

\begin{figure}[t]
\centering
\includegraphics[width=0.9\textwidth]{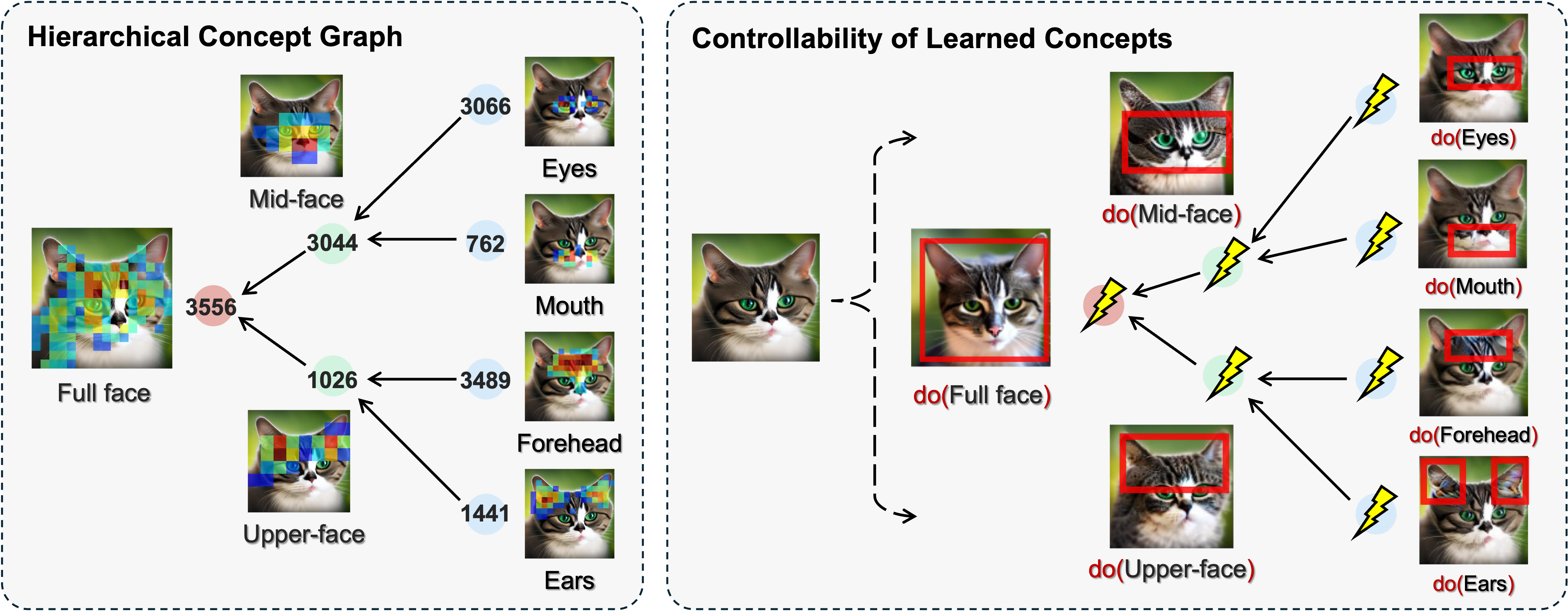}
\caption{
    \small
    \textbf{Examples of hierarchical concept graphs for text-to-image models.} Our method successfully recovers meaningful hierarchical structures, where each node encodes distinct semantic concepts. On the right, we demonstrate feature steering, where manipulating individual nodes leads to changes in the output that align with their position in the hierarchy.
    Intervening on a high-level concept in the learned graph (``Full face'') alters the cat's entire facial structure and fur pattern. In contrast, intervening on a learned lower-level concept (e.g., ``Eye'') produces a much more localized edit, changing only the shape and color of the eyes while leaving the rest of the face intact. More examples in Appendix~\ref{app:empirical_results}.
}
\vspace{-0.2cm}
\label{fig:vision_causal}
\end{figure}

\begin{table}[t!]
\centering
\setlength{\tabcolsep}{3.6pt}
\begin{tabular}{lcccccccccc}
\hline
\textbf{Method} & \textbf{I2P $\downarrow$} & \multicolumn{4}{c}{\textbf{RING-A-BELL $\downarrow$}} & \textbf{P4D $\downarrow$} & \textbf{UATK $\downarrow$} & \multicolumn{2}{c}{\textbf{COCO}} \\
\cline{3-6} \cline{9-10}
 &  & K77 & K38 & K16 & AVG &  &  & FID $\downarrow$ & CLIP $\uparrow$ \\
\hline
SD 1.4  & 17.8  & 85.26 & 87.37 & 93.68 & 88.10 & 98.70 & 69.70 & 16.71 & \textbf{31.3} \\
ESD & 2.87  & 20.00 & 29.47 & 35.79 & 28.42 & 15.49 & 2.87  & 18.18 & 30.2 \\
SA & 2.81  & 63.15 & 56.84 & 56.84 & 58.94 & 12.68 & 2.81  & 25.80 & 29.7 \\
CA & 1.04  & 86.32 & 91.69 & 94.26 & 90.76 & 5.63  & 1.04  & 24.12 & 30.1 \\
MACE  & 1.51  & 2.10  & \textbf{0.00}  & \textbf{0.00}  & \textbf{0.70}  & 2.82  & 1.51  & \textbf{16.80} & 28.7 \\
UCE & 0.87  & 10.52 & 9.47  & 12.61 & 10.87 & 9.86  & 0.87  & 17.99 & 30.2 \\
RECE & 0.72  & 5.26  & 4.21  & 5.26  & 4.91  & 5.63  & 0.72  & 17.74 & 30.2 \\
SDID  & 3.77  & 94.74 & 95.79 & 90.53 & 93.68 & 69.54 & 30.99 & 22.16 & 31.1 \\
SLD-MAX & 1.74  & 23.16 & 32.63 & 42.11 & 32.63 & 9.14  & 2.44  & 28.75 & 28.4 \\
SLD-STRONG   & 2.28  & 56.84 & 64.21 & 61.05 & 60.70 & 33.10 & 3.10  & 24.40 & 29.1 \\
SLD-MEDIUM    & 3.95  & 92.63 & 88.42 & 91.05 & 90.70 & 24.00 & 1.98  & 21.17 & 29.8 \\
SD1.4-NegPrompt & 0.74  & 17.89 & 40.42 & 34.74 & 31.68 & 10.00 & 1.46  & 18.33 & 30.1 \\
SAFREE        & 1.45  & 35.78 & 47.36 & 55.78 & 46.31 & 10.56 & 1.45  & 19.32 & 30.1 \\
TRASCE        & 0.45  & 1.05  & 2.10  & 2.10  & 1.75  & 3.97  & \textbf{0.70}  & 17.41 & 29.9 \\
ConceptSteer    & 0.36  & 3.16  & 8.42  & 9.47  & 7.02  & 1.99  & 2.11  & 18.67 & 30.8 \\
\hline
\textbf{Ours}        & \textbf{0.25} & \textbf{1.05} & \textbf{0.00} & 2.11 & {1.05} & \textbf{0.66} & 2.11 & 17.02 & \textbf{31.3} \\ \hline
\end{tabular}
\caption{\textbf{Model unlearning comparisons}. Our method delivers competitive results on unlearning tasks without compromising standard text-to-image generation. See Appendix~\ref{app:dicussions} for details.}
\label{tab:unlearning}
\vspace{-0.3cm}
\end{table}

\begin{figure}[t]
{
\setlength{\belowcaptionskip}{-5pt}
    \centering
    \includegraphics[width=\linewidth]{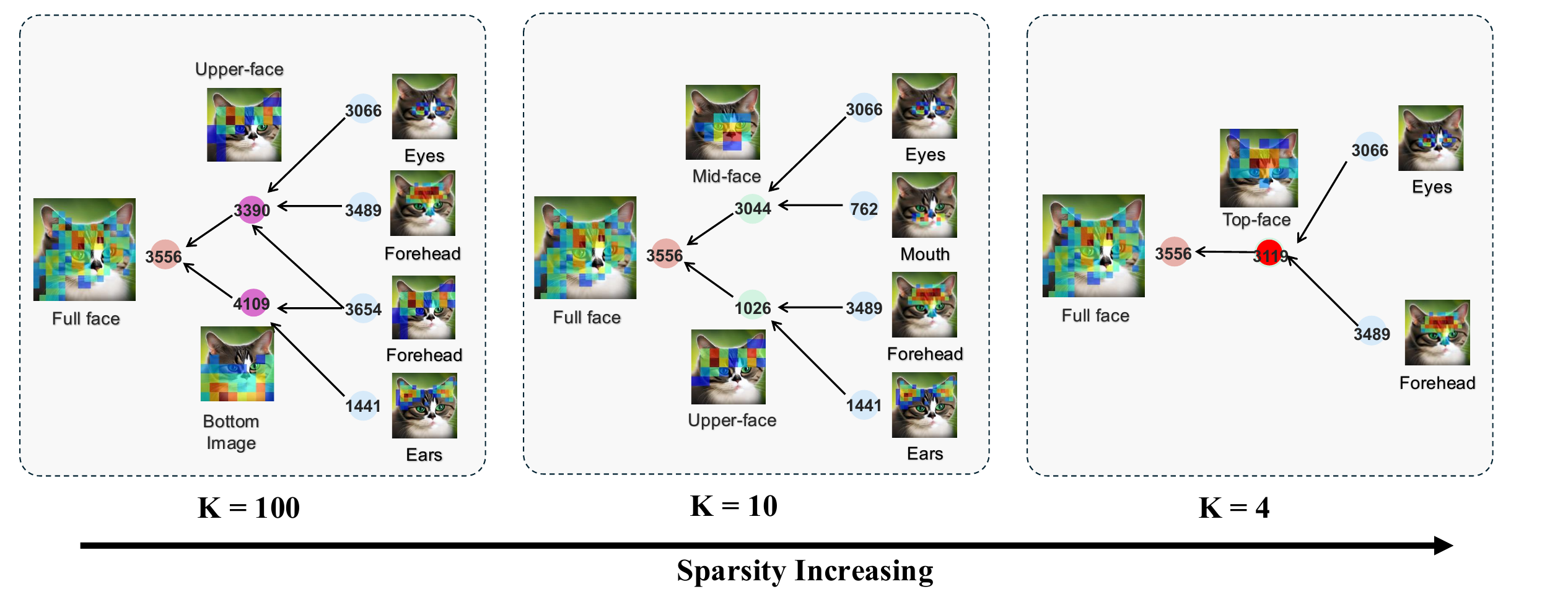}
    \caption{
        \small
        \textbf{Ablation studies on the sparsity constraint}. 
        We control feature sparsity at timestep 500 by varying the number of features (K) in SAE. Without enforcing sparsity, the resulting concepts tend to be dense, and the features are less interpretable. Conversely, higher sparsity leads to a more interpretable, sparser graph. However, when sparsity becomes too high, the resulting graph may become overly sparse and fail to adequately capture the generation of the cat face. }
    \label{fig:ablation}}
\end{figure}

\begin{table}[h]
\centering
\vspace{0.5em}
\begin{tabular}{c c c c c c c c c c}
\toprule
\multicolumn{5}{c}{\textbf{(a) Spatial activation spread}} & \multicolumn{5}{c}{\textbf{(b) SAE-deactivated generation comparison}} \\
\cmidrule(lr){1-5} \cmidrule(lr){6-10}
\textbf{Timestep} & \textbf{Top1} & \textbf{Top3} & \textbf{Top5} & \textbf{Top10} &
\textbf{L1} & \textbf{LPIPS} & \textbf{CLIP} & \textbf{DINO} \\
\midrule
100 & 0.27 & 0.21 & 0.19 & 0.15 & 0.004 & 0.002 & 0.999 & 0.999 \\
500 & 0.30 & 0.25 & 0.21 & 0.17 & 0.013 & 0.020 & 0.995 & 0.993 \\
899 & 0.53 & 0.41 & 0.33 & 0.24 & 0.070 & 0.220 & 0.948 & 0.903 \\
\bottomrule
\end{tabular}
\caption{\textbf{Quantitative analyses across different noise levels.} 
(a) Spatial activation spread: average proportion of pixels influenced by the top-$k$ SAE activations. Higher timesteps affect a larger spatial area, indicating that SAEs at noisier steps capture more global, distributed concepts. 
(b) SAE-deactivated generation comparison: similarity metrics between original and SAE-deactivated images. Deactivation at higher timesteps produces greater perceptual and semantic changes, supporting the presence of a hierarchical organization of concepts across timesteps.}

\label{tab:hierarchical_analysis}
\end{table}

\subsection{Interpretability Analysis} 
\label{subsec:interpretation_exp}

\paragraph{Hierarchical concept graph.}
Figure~\ref{fig:vision_causal} illustrates a hierarchical graph learned through our approach (more in Appendix~\ref{app:empirical_results}). On the left, we display activation maps of different SAE features. 
Brown nodes (SAE nodes trained on timestep 899) capture high-level features, such as node 3556 representing an entire cat face. Green nodes (timestep 500) reflect mid-level features, like node 3044 capturing the central face. Blue nodes (timestep 100) capture fine details—node 3066 activates on the eyes and node 762 on the mouth. This demonstrates a clear progression from coarse to fine-grained concepts across timesteps. 
To thoroughly examine the existence of the hierarchical concept graph, we conduct two complementary experiments demonstrating that activations at higher timesteps capture more global semantics, while those at lower timesteps capture more localized details.
First, we quantify the spatial spread of activations across timesteps. For each SAE, we compute attribution maps for its top feature indices. Given an SAE feature of shape $64 \times 64 \times 5120$, we compute a $64 \times 64$ attribution map. Applying a 0.1 threshold yields a binary attribution map, from which we measure the proportion of activated pixels. 
Across 1,000 samples, approximately 280, 630, 880, and 1,400 unique concepts are activated for $K=1,3,5,$ and $10$, respectively. 
Activations at timestep 899 influence a larger spatial area, indicating that higher timesteps capture more global, distributed concepts.
Second, we generate images from 10,000 COCO prompts and deactivate the top-1 SAE activation at each timestep. Comparing the modified generations with the originals shows that deactivations at noisier timesteps cause substantial, global changes, while those at less noisy timesteps produce localized effects. These results confirm that features at different noise levels encode distinct abstraction levels, supporting the hierarchical concept graph.

\paragraph{Concept steering in hierarchical graphs.}
We conduct concept steering using our discovered features, as shown on the right side of Fig.~\ref{fig:vision_causal} (more in Appendix~\ref{app:empirical_results}). 
Given a model intermediate feature $x$, the SAE encoder $E$ and decoder $D$ are trained to reconstruct $x$. 
To steer a specific concept, we obtain the latent representation $z=E(x)$, and extract the steering vector $v$ corresponding to the desired feature. We then modify the original feature to create a steered version $x^\prime =x+\lambda D(v)$, where $\lambda$ modulates the strength. By feeding the steered $x^\prime$ back into the diffusion process at the same timestep, we generate images that reflect the influence of the selected concept. For example, steering node 3556 -- associated with the entire face of a cat -- results in a significantly altered cat face. Steering the green node 1026 modifies only the upper part of the face, illustrating that it encodes localized information specific to that region.

\begin{table}[t]
\centering
\begin{tabular}{lccc}
\hline
\textbf{Metric} & \textbf{SD 1.4} & \textbf{SD1.4 (SAE w/o hier.)} & \textbf{SD1.4 (Ours)} \\
\hline
Add tabby pattern – CLIP-I $\downarrow$ & 0.91 $\pm$ 0.05 & 0.83 $\pm$ 0.07 & \textbf{0.93 $\pm$ 0.04} \\
Add tabby pattern – CLIP-T $\uparrow$ & 0.27 $\pm$ 0.00 & \textbf{0.28 $\pm$ 0.02} & \textbf{0.28 $\pm$ 0.01} \\
Add mountains – CLIP-I $\downarrow$ & 0.84 $\pm$ 0.06 & 0.83 $\pm$ 0.04 & \textbf{0.91 $\pm$ 0.03} \\
Add mountains – CLIP-T $\uparrow$ & \textbf{0.33 $\pm$ 0.01} & 0.32 $\pm$ 0.01 & \textbf{0.33 $\pm$ 0.01} \\
Replace rock w/ stump – CLIP-I $\downarrow$ & 0.93 $\pm$ 0.02 & 0.95 $\pm$ 0.02 & \textbf{0.96 $\pm$ 0.02} \\
Replace rock w/ stump – CLIP-T $\uparrow$ & \textbf{0.31 $\pm$ 0.01} & 0.29 $\pm$ 0.01 & \textbf{0.31 $\pm$ 0.01} \\
\hline
\end{tabular}
\caption{\textbf{Controllable image generation results}. Our method achieves the best CLIP-I metric, demonstrating greater fidelity to the input images, while reliably executing the target edits.}
\label{tab:controll_image_gen}
\vspace{-0.3cm}
\end{table}

\paragraph{Ablation.}
As established in the theoretical framework, sparsity is crucial for identifiability. To empirically validate this, we visualize the resulting causal graphs under varying levels of sparsity, as shown in Fig.~\ref{fig:ablation} (more in Appendix~\ref{app:empirical_results}). When sparsity is not enforced, the resulting graph becomes overly dense, making it difficult to interpret and diminishing its semantic clarity. Conversely, imposing excessive sparsity leads to an overly pruned graph that lacks sufficient structure to meaningfully explain the generation process, such as in the case of the cat image. These observations highlight the importance of balancing sparsity to preserve interpretability while maintaining explanatory power.

\begin{figure}[t]
    \centering
    \includegraphics[width=0.86\linewidth]{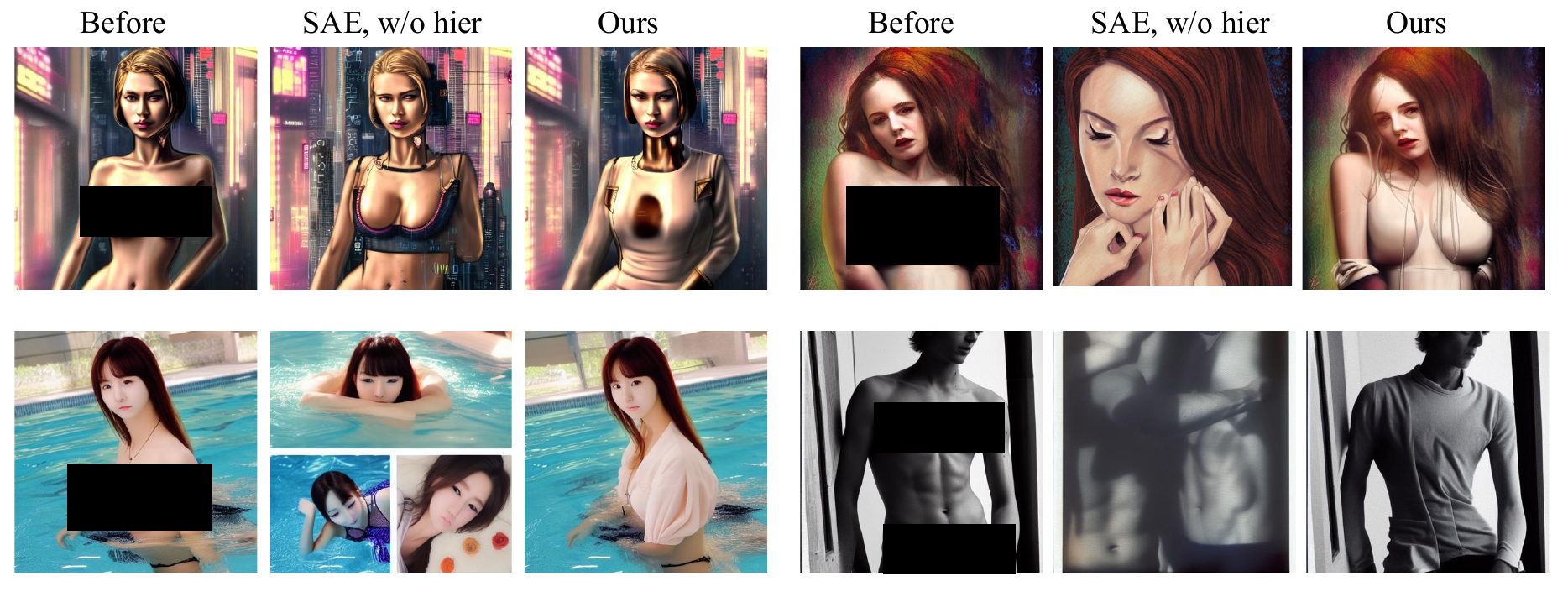}
    \caption{
    \textbf{Generated samples with P4D prompts~\citep{chin2023prompting4debugging}.} 
    The Stable Diffusion model is vulnerable to the prompts in the p4d dataset, producing unsafe images. When the hierarchical relationship across timesteps is not considered, negative steering with SAE results in drastic changes to the output. In contrast, our method learns to apply modifications to the nudity feature at a suitable timestep without introducing additional distortions.}
    \vspace{-0.5cm}
    \label{fig:concept_removal}
\end{figure}

\begin{figure}[ht]
    \centering
    \includegraphics[width=0.9\linewidth]{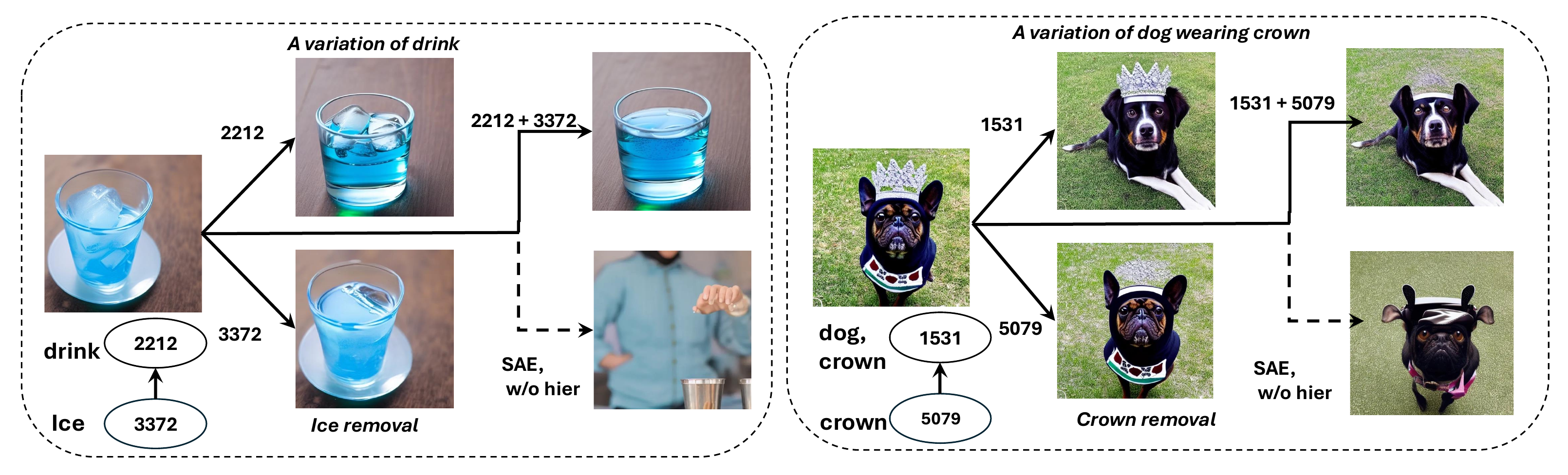}
    \caption{\textbf{Examples of multi-level editing} (best viewed with zoom).
High-level node 2212 contains all information about the cup, while mid-level node 3372 focuses primarily on the ice cubes. Similarly, high-level node 1531 encompasses all information about the dog (including the crown), and mid-level node 5079 is dedicated to the crown.
By modeling hierarchical relationships, we can perform edits that are often difficult to achieve with a single-layer edit. For instance, if we want to generate a variation of the cup while removing the ice cubes, we can apply feature steering on high-level node 2212 to create a new version of the cup, and simultaneously apply negative feature steering on mid-level node 3372 to remove the ice cubes.}
    \vspace{-0.5cm}
    \label{fig:multi-level}
\end{figure}

\subsection{Downstream Tasks} \label{subsec:downstream_exp}
Thanks to our theoretical framework, we can naturally perform a range of image generation and editing tasks, including model unlearning, controllable image generation, and multi-level editing.

\paragraph{Model unlearning.}
 We provide quantitative results of model unlearning on four benchmark datasets: IP2P~\citep{schramowski2023safe}, three splits of RING-A-BELL~\citep{tsai2023ring}, P4D~\citep{chin2023prompting4debugging}, and UnlearnDiffATK~\citep{zhang2024generate}. 
These benchmarks focus on removing nudity-related concepts, and we report the accuracy of a pretrained nudity detector. Our method achieves the best results across all benchmarks.
In addition, to assess whether our method preserves general text-to-image capability, we apply feature steering on normal prompts from MSCOCO~\citep{lin2014microsoft}. The 10K results, reflected in low FID and high CLIP scores, demonstrate that our method successfully identifies and removes nudity concepts without affecting unrelated concepts. We also provide results on style removal in the appendix (Table~\ref{tab:style_transfer}) and we achieve superior performance across different metrics and tasks.

\textbf{Controllable image generation.} 
We also evaluate controllable image generation on three editing tasks: adding tabby patterns to cat faces,
adding mountains to landscape images, and replacing rocks with textured tree stumps. As shown in Table~\ref{tab:controll_image_gen} and Fig.\ref{fig:panda_hier}, our method achieves superior results compared to both the standard text-guided model and SAE without hierarchical modeling.  

\textbf{Multi-level image editing.} 
A key advantage of the hierarchical concept graph is that it can combine nodes across different levels for fine-grained image editing. In Fig.~\ref{fig:multi-level}, to obtain a new drink without ice (while preserving the background), we can apply multi-level editing by steering features at both high-level node 2212 and mid-level node 3372 simultaneously. Without such hierarchical relationship modeling, conventional methods struggle to produce this combination, which can result in undesired changes such as the drink being replaced by a person or the dog’s background.


\section{Conclusion}

In this work, we present a theoretical framework using causal minimality for identifying latent concepts in hierarchical selection models. We prove that generative model representations can map to true latent variables. Empirically, applying these constraints enables extracting meaningful hierarchical concept graphs from leading models, enhancing interpretability and grounded control.
\textbf{Limitations:} Our identifiability results are derived under standard regularity/minimality assumptions. These conditions are consistent with the inductive biases we encourage and are supported by our empirical findings. Extending the guarantees to broader forms of misspecification remains an important direction for future work.

\clearpage


\bibliography{references}

@inproceedings{rombach2022high,
  title={High-resolution image synthesis with latent diffusion models},
  author={Rombach, Robin and Blattmann, Andreas and Lorenz, Dominik and Esser, Patrick and Ommer, Bj{\"o}rn},
  booktitle={Proceedings of the IEEE/CVF conference on computer vision and pattern recognition},
  pages={10684--10695},
  year={2022}
}

@inproceedings{gandikota2023erasing,
  title={Erasing concepts from diffusion models},
  author={Gandikota, Rohit and Materzynska, Joanna and Fiotto-Kaufman, Jaden and Bau, David},
  booktitle={Proceedings of the IEEE/CVF international conference on computer vision},
  pages={2426--2436},
  year={2023}
}

@article{heng2023selective,
  title={Selective amnesia: A continual learning approach to forgetting in deep generative models},
  author={Heng, Alvin and Soh, Harold},
  journal={Advances in Neural Information Processing Systems},
  volume={36},
  pages={17170--17194},
  year={2023}
}

@inproceedings{kumari2023ablating,
  title={Ablating concepts in text-to-image diffusion models},
  author={Kumari, Nupur and Zhang, Bingliang and Wang, Sheng-Yu and Shechtman, Eli and Zhang, Richard and Zhu, Jun-Yan},
  booktitle={Proceedings of the IEEE/CVF International Conference on Computer Vision},
  pages={22691--22702},
  year={2023}
}

@inproceedings{lu2024mace,
  title={Mace: Mass concept erasure in diffusion models},
  author={Lu, Shilin and Wang, Zilan and Li, Leyang and Liu, Yanzhu and Kong, Adams Wai-Kin},
  booktitle={Proceedings of the IEEE/CVF Conference on Computer Vision and Pattern Recognition},
  pages={6430--6440},
  year={2024}
}

@inproceedings{gandikota2024unified,
  title={Unified concept editing in diffusion models},
  author={Gandikota, Rohit and Orgad, Hadas and Belinkov, Yonatan and Materzy{\'n}ska, Joanna and Bau, David},
  booktitle={Proceedings of the IEEE/CVF Winter Conference on Applications of Computer Vision},
  pages={5111--5120},
  year={2024}
}

@inproceedings{gong2024reliable,
  title={Reliable and efficient concept erasure of text-to-image diffusion models},
  author={Gong, Chao and Chen, Kai and Wei, Zhipeng and Chen, Jingjing and Jiang, Yu-Gang},
  booktitle={European Conference on Computer Vision},
  pages={73--88},
  year={2024},
  organization={Springer}
}

@inproceedings{li2024self,
  title={Self-discovering interpretable diffusion latent directions for responsible text-to-image generation},
  author={Li, Hang and Shen, Chengzhi and Torr, Philip and Tresp, Volker and Gu, Jindong},
  booktitle={Proceedings of the IEEE/CVF Conference on Computer Vision and Pattern Recognition},
  pages={12006--12016},
  year={2024}
}

@inproceedings{schramowski2023safe,
  title={Safe latent diffusion: Mitigating inappropriate degeneration in diffusion models},
  author={Schramowski, Patrick and Brack, Manuel and Deiseroth, Bj{\"o}rn and Kersting, Kristian},
  booktitle={Proceedings of the IEEE/CVF Conference on Computer Vision and Pattern Recognition},
  pages={22522--22531},
  year={2023}
}

@article{yoon2024safree,
  title={Safree: Training-free and adaptive guard for safe text-to-image and video generation},
  author={Yoon, Jaehong and Yu, Shoubin and Patil, Vaidehi and Yao, Huaxiu and Bansal, Mohit},
  journal={arXiv preprint arXiv:2410.12761},
  year={2024}
}

@article{jain2024trasce,
  title={Trasce: Trajectory steering for concept erasure},
  author={Jain, Anubhav and Kobayashi, Yuya and Shibuya, Takashi and Takida, Yuhta and Memon, Nasir and Togelius, Julian and Mitsufuji, Yuki},
  journal={arXiv preprint arXiv:2412.07658},
  year={2024}
}

@article{tsai2023ring,
  title={Ring-a-bell! how reliable are concept removal methods for diffusion models?},
  author={Tsai, Yu-Lin and Hsu, Chia-Yi and Xie, Chulin and Lin, Chih-Hsun and Chen, Jia-You and Li, Bo and Chen, Pin-Yu and Yu, Chia-Mu and Huang, Chun-Ying},
  journal={arXiv preprint arXiv:2310.10012},
  year={2023}
}

@article{chin2023prompting4debugging,
  title={Prompting4debugging: Red-teaming text-to-image diffusion models by finding problematic prompts},
  author={Chin, Zhi-Yi and Jiang, Chieh-Ming and Huang, Ching-Chun and Chen, Pin-Yu and Chiu, Wei-Chen},
  journal={arXiv preprint arXiv:2309.06135},
  year={2023}
}

@inproceedings{zhang2024generate,
  title={To generate or not? safety-driven unlearned diffusion models are still easy to generate unsafe images... for now},
  author={Zhang, Yimeng and Jia, Jinghan and Chen, Xin and Chen, Aochuan and Zhang, Yihua and Liu, Jiancheng and Ding, Ke and Liu, Sijia},
  booktitle={European Conference on Computer Vision},
  pages={385--403},
  year={2024},
  organization={Springer}
}

@inproceedings{lin2014microsoft,
  title={Microsoft coco: Common objects in context},
  author={Lin, Tsung-Yi and Maire, Michael and Belongie, Serge and Hays, James and Perona, Pietro and Ramanan, Deva and Doll{\'a}r, Piotr and Zitnick, C Lawrence},
  booktitle={European conference on computer vision},
  pages={740--755},
  year={2014},
  organization={Springer}
}

@article{shwartz2017opening,
  title={Opening the black box of deep neural networks via information},
  author={Shwartz-Ziv, Ravid and Tishby, Naftali},
  journal={arXiv preprint arXiv:1703.00810},
  year={2017}
}

@misc{brown2020language,
      title={Language Models are Few-Shot Learners}, 
      author={Tom B. Brown and Benjamin Mann and Nick Ryder and Melanie Subbiah and Jared Kaplan and Prafulla Dhariwal and Arvind Neelakantan and Pranav Shyam and Girish Sastry and Amanda Askell and Sandhini Agarwal and Ariel Herbert-Voss and Gretchen Krueger and Tom Henighan and Rewon Child and Aditya Ramesh and Daniel M. Ziegler and Jeffrey Wu and Clemens Winter and Christopher Hesse and Mark Chen and Eric Sigler and Mateusz Litwin and Scott Gray and Benjamin Chess and Jack Clark and Christopher Berner and Sam McCandlish and Alec Radford and Ilya Sutskever and Dario Amodei},
      year={2020},
      eprint={2005.14165},
      archivePrefix={arXiv},
      primaryClass={cs.CL}
}

@article{khemakhem2020icebeem,
  title={Ice-beem: Identifiable conditional energy-based deep models based on nonlinear ica},
  author={Khemakhem, Ilyes and Monti, Ricardo and Kingma, Diederik and Hyvarinen, Aapo},
  journal={Advances in Neural Information Processing Systems},
  volume={33},
  pages={12768--12778},
  year={2020}
}

@inproceedings{khemakhem2020variational,
  title={Variational autoencoders and nonlinear ica: A unifying framework},
  author={Khemakhem, Ilyes and Kingma, Diederik and Monti, Ricardo and Hyvarinen, Aapo},
  booktitle={International Conference on Artificial Intelligence and Statistics},
  pages={2207--2217},
  year={2020},
  organization={PMLR}
}

@inproceedings{hyvarinen2019nonlinear,
  title={Nonlinear ICA using auxiliary variables and generalized contrastive learning},
  author={Hyvarinen, Aapo and Sasaki, Hiroaki and Turner, Richard},
  booktitle={The 22nd International Conference on Artificial Intelligence and Statistics},
  pages={859--868},
  year={2019},
  organization={PMLR}
}

@article{hyvarinen2016unsupervised,
  title={Unsupervised feature extraction by time-contrastive learning and nonlinear ica},
  author={Hyvarinen, Aapo and Morioka, Hiroshi},
  journal={Advances in neural information processing systems},
  volume={29},
  year={2016}
}

@article{moran2021identifiable,
  title={Identifiable Variational Autoencoders via Sparse Decoding},
  author={Moran, Gemma E and Sridhar, Dhanya and Wang, Yixin and Blei, David M},
  journal={arXiv preprint arXiv:2110.10804},
  year={2021}
}

@inproceedings{anandkumar2013learning,
  title={Learning linear bayesian networks with latent variables},
  author={Anandkumar, Animashree and Hsu, Daniel and Javanmard, Adel and Kakade, Sham},
  booktitle={International Conference on Machine Learning},
  pages={249--257},
  year={2013},
  organization={PMLR}
}

@article{zheng2022identifiability,
  title={On the Identifiability of Nonlinear ICA: Sparsity and Beyond},
  author={Zheng, Yujia and Ng, Ignavier and Zhang, Kun},
  journal={arXiv preprint arXiv:2206.07751},
  year={2022}
}

@inproceedings{xie2022identification,
  title={Identification of linear non-gaussian latent hierarchical structure},
  author={Xie, Feng and Huang, Biwei and Chen, Zhengming and He, Yangbo and Geng, Zhi and Zhang, Kun},
  booktitle={International Conference on Machine Learning},
  pages={24370--24387},
  year={2022},
  organization={PMLR}
}

@article{huang2022latent,
  title={Latent hierarchical causal structure discovery with rank constraints},
  author={Huang, Biwei and Low, Charles Jia Han and Xie, Feng and Glymour, Clark and Zhang, Kun},
  journal={Advances in Neural Information Processing Systems},
  volume={35},
  pages={5549--5561},
  year={2022}
}

@inproceedings{ramesh2021zero,
  title={Zero-shot text-to-image generation},
  author={Ramesh, Aditya and Pavlov, Mikhail and Goh, Gabriel and Gray, Scott and Voss, Chelsea and Radford, Alec and Chen, Mark and Sutskever, Ilya},
  booktitle={International Conference on Machine Learning},
  pages={8821--8831},
  year={2021},
  organization={PMLR}
}

@article{olah2020zoom,
  author = {Olah, Chris and Cammarata, Nick and Schubert, Ludwig and Goh, Gabriel and Petrov, Michael and Carter, Shan},
  title = {Zoom In: An Introduction to Circuits},
  journal = {Distill},
  year = {2020},
  note = {https://distill.pub/2020/circuits/zoom-in},
  doi = {10.23915/distill.00024.001}
}

@article{choi2011learning,
  title={Learning latent tree graphical models},
  author={Choi, Myung Jin and Tan, Vincent YF and Anandkumar, Animashree and Willsky, Alan S},
  journal={Journal of Machine Learning Research},
  volume={12},
  pages={1771--1812},
  year={2011},
  publisher={Journal of Machine Learning Research}
}

@BOOK{Pearl88,
  AUTHOR =       "J. Pearl",
  TITLE =        "Probabilistic Reasoning in Intelligent Systems: Networks of Plausible Inference",
  PUBLISHER =    "Morgan Kaufmann",
  YEAR =         "1988",
}

@article{kaddour2022causal,
  title={Causal machine learning: A survey and open problems},
  author={Kaddour, Jean and Lynch, Aengus and Liu, Qi and Kusner, Matt J and Silva, Ricardo},
  journal={arXiv preprint arXiv:2206.15475},
  year={2022}
}

@inproceedings{kong2022partial,
  title={Partial disentanglement for domain adaptation},
  author={Kong, Lingjing and Xie, Shaoan and Yao, Weiran and Zheng, Yujia and Chen, Guangyi and Stojanov, Petar and Akinwande, Victor and Zhang, Kun},
  booktitle={International Conference on Machine Learning},
  pages={11455--11472},
  year={2022},
  organization={PMLR}
}

@article{lachapelle2023additive,
  title={Additive decoders for latent variables identification and cartesian-product extrapolation},
  author={Lachapelle, S{\'e}bastien and Mahajan, Divyat and Mitliagkas, Ioannis and Lacoste-Julien, Simon},
  journal={Advances in Neural Information Processing Systems},
  volume={36},
  year={2024}
}

@inproceedings{gu2023bayesian,
      title={Bayesian Pyramids: Identifiable Multilayer Discrete Latent Structure Models for Discrete Data}, 
      author={Yuqi Gu and David B. Dunson},
      booktitle={Journal of the Royal Statistical Society Series B: Statistical Methodology},
      year={2023},
      eprint={2101.10373},
      archivePrefix={arXiv},
      primaryClass={stat.ME}
}

@inproceedings{ramesh2022hierarchical,
      title={Hierarchical Text-Conditional Image Generation with CLIP Latents}, 
      author={Aditya Ramesh and Prafulla Dhariwal and Alex Nichol and Casey Chu and Mark Chen},
      booktitle={Advances in Neural Information Processing Systems 36 (NeurIPS 2022)},
      year={2022},
      eprint={2204.06125},
      archivePrefix={arXiv},
      primaryClass={cs.CV}
}

@inproceedings{rombach2021highresolution,
      title={High-Resolution Image Synthesis with Latent Diffusion Models}, 
      author={Robin Rombach and Andreas Blattmann and Dominik Lorenz and Patrick Esser and Björn Ommer},
      booktitle={Proceedings of the 2022 IEEE/CVF Conference on Computer Vision and Pattern Recognition (CVPR 2022)},
      year={2022},
      eprint={2112.10752},
      archivePrefix={arXiv},
      primaryClass={cs.CV}
}

@inproceedings{dong2023versatile,
  title={A Versatile Causal Discovery Framework to Allow Causally-Related Hidden Variables},
  author={Dong, Xinshuai and Huang, Biwei and Ng, Ignavier and Song, Xiangchen and Zheng, Yujia and Jin, Songyao and Legaspi, Roberto and Spirtes, Peter and Zhang, Kun},
  booktitle={The Twelfth International Conference on Learning Representations},
  year={2023}
}

@article{kong2023identification,
  title={Identification of nonlinear latent hierarchical models},
  author={Kong, Lingjing and Huang, Biwei and Xie, Feng and Xing, Eric and Chi, Yuejie and Zhang, Kun},
  journal={Advances in Neural Information Processing Systems},
  volume={36},
  year={2023}
}

@InProceedings{Kong2023understanding,
    author    = {Kong, Lingjing and Ma, Martin Q. and Chen, Guangyi and Xing, Eric P. and Chi, Yuejie and Morency, Louis-Philippe and Zhang, Kun},
    title     = {Understanding Masked Autoencoders via Hierarchical Latent Variable Models},
    booktitle = {Proceedings of the IEEE/CVF Conference on Computer Vision and Pattern Recognition (CVPR)},
    month     = {June},
    year      = {2023},
    pages     = {7918-7928}
}

@inproceedings{arora2012learning,
  title={Learning topic models--going beyond SVD},
  author={Arora, Sanjeev and Ge, Rong and Moitra, Ankur},
  booktitle={2012 IEEE 53rd annual symposium on foundations of computer science},
  pages={1--10},
  year={2012},
  organization={IEEE}
}

@inproceedings{arora2012practical,
  title={A practical algorithm for topic modeling with provable guarantees},
  author={Arora, Sanjeev and Ge, Rong and Halpern, Yonatan and Mimno, David and Moitra, Ankur and Sontag, David and Wu, Yichen and Zhu, Michael},
  booktitle={International conference on machine learning},
  pages={280--288},
  year={2013},
  organization={PMLR}
}

@inproceedings{ho2020denoising,
      title={Denoising Diffusion Probabilistic Models}, 
      author={Jonathan Ho and Ajay Jain and Pieter Abbeel},
      booktitle={Advances in Neural Information Processing Systems 33 (NeurIPS 2020)},
      year={2020},
      eprint={2006.11239},
      archivePrefix={arXiv},
      primaryClass={cs.LG}
}

@inproceedings{jahanian2020steerability,
  title={On the" steerability" of generative adversarial networks},
  author={Jahanian, Ali and Chai, Lucy and Isola, Phillip},
  booktitle={International Conference on Learning Representations},
  year={2019}
}

@inproceedings{shen2020interpreting,
  title={Interpreting the latent space of gans for semantic face editing},
  author={Shen, Yujun and Gu, Jinjin and Tang, Xiaoou and Zhou, Bolei},
  booktitle={Proceedings of the IEEE/CVF conference on computer vision and pattern recognition},
  pages={9243--9252},
  year={2020}
}

@inproceedings{wu2020stylespace,
  title={Stylespace analysis: Disentangled controls for stylegan image generation},
  author={Wu, Zongze and Lischinski, Dani and Shechtman, Eli},
  booktitle={Proceedings of the IEEE/CVF conference on computer vision and pattern recognition},
  pages={12863--12872},
  year={2021}
}

@inproceedings{sohldickstein2015deep,
      title={Deep Unsupervised Learning using Nonequilibrium Thermodynamics}, 
      author={Jascha Sohl-Dickstein and Eric A. Weiss and Niru Maheswaranathan and Surya Ganguli},
      booktitle={Proceedings of the International Conference on Machine Learning (ICML 2015)},
      year={2015},
      eprint={1503.03585},
      archivePrefix={arXiv},
      primaryClass={cs.LG}
}

@inproceedings{song2022denoising,
      title={Denoising Diffusion Implicit Models}, 
      author={Jiaming Song and Chenlin Meng and Stefano Ermon},
      booktitle={International Conference on Learning Representations (ICLR 2022)},
      year={2022},
      eprint={2010.02502},
      archivePrefix={arXiv},
      primaryClass={cs.LG}
}

@inproceedings{dhariwal2021diffusion,
      title={Diffusion Models Beat GANs on Image Synthesis}, 
      author={Prafulla Dhariwal and Alex Nichol},
      booktitle={Advances in Neural Information Processing Systems 34 (NeurIPS 2021)},
      year={2021},
      eprint={2105.05233},
      archivePrefix={arXiv},
      primaryClass={cs.LG}
}

@inproceedings{nichol2021improved,
      title={Improved Denoising Diffusion Probabilistic Models}, 
      author={Alex Nichol and Prafulla Dhariwal},
      booktitle={Proceedings of the International Conference on Machine Learning (ICML 2021)},
      year={2021},
      eprint={2102.09672},
      archivePrefix={arXiv},
      primaryClass={cs.LG}
}

@article{zhang2004hierarchical,
  title={Hierarchical latent class models for cluster analysis},
  author={Zhang, Nevin L},
  journal={The Journal of Machine Learning Research},
  volume={5},
  pages={697--723},
  year={2004},
  publisher={JMLR. org}
}

@book{pearl2009causality,
  title={Causality},
  author={Pearl, Judea},
  year={2009},
  publisher={Cambridge university press}
}

@book{spirtes2001causation,
  title={Causation, prediction, and search},
  author={Spirtes, Peter and Glymour, Clark and Scheines, Richard},
  year={2001},
  publisher={MIT press}
}

@inproceedings{
kong2024learning,
title={Learning Discrete Concepts in Latent Hierarchical Models},
author={Lingjing Kong and Guangyi Chen and Biwei Huang and Eric P. Xing and Yuejie Chi and Kun Zhang},
booktitle={The Thirty-eighth Annual Conference on Neural Information Processing Systems},
year={2024},
url={https://openreview.net/forum?id=bO5bUxvH6m}
}

@inproceedings{
zhang2024causal,
title={Causal Representation Learning from Multiple Distributions: A General Setting},
author={Kun Zhang and Shaoan Xie and Ignavier Ng and Yujia Zheng},
booktitle={Forty-first International Conference on Machine Learning},
year={2024},
url={https://openreview.net/forum?id=Pte6iiXvpf}
}

@inproceedings{brady2023provably,
  title={Provably learning object-centric representations},
  author={Brady, Jack and Zimmermann, Roland S and Sharma, Yash and Sch{\"o}lkopf, Bernhard and Von K{\"u}gelgen, Julius and Brendel, Wieland},
  booktitle={International Conference on Machine Learning},
  pages={3038--3062},
  year={2023},
  organization={PMLR}
}

@article{raffel2020exploring,
  title={Exploring the limits of transfer learning with a unified text-to-text transformer},
  author={Raffel, Colin and Shazeer, Noam and Roberts, Adam and Lee, Katherine and Narang, Sharan and Matena, Michael and Zhou, Yanqi and Li, Wei and Liu, Peter J},
  journal={Journal of machine learning research},
  volume={21},
  number={140},
  pages={1--67},
  year={2020}
}

@article{von2021self,
  title={Self-supervised learning with data augmentations provably isolates content from style},
  author={Von K{\"u}gelgen, Julius and Sharma, Yash and Gresele, Luigi and Brendel, Wieland and Sch{\"o}lkopf, Bernhard and Besserve, Michel and Locatello, Francesco},
  journal={Advances in neural information processing systems},
  volume={34},
  pages={16451--16467},
  year={2021}
}

@book{spirtes2000causation,
	title={Causation, Prediction, and Search},
	author={Spirtes, Peter and Glymour, Clark N and Scheines, Richard},
	year={2000},
	publisher={MIT press}
}

@article{reizinger2024cross,
  title={Cross-Entropy Is All You Need To Invert the Data Generating Process},
  author={Reizinger, Patrik and Bizeul, Alice and Juhos, Attila and Vogt, Julia E and Balestriero, Randall and Brendel, Wieland and Klindt, David},
  journal={arXiv preprint arXiv:2410.21869},
  year={2024}
}

@article{park2024geometry,
  title={The geometry of categorical and hierarchical concepts in large language models},
  author={Park, Kiho and Choe, Yo Joong and Jiang, Yibo and Veitch, Victor},
  journal={arXiv preprint arXiv:2406.01506},
  year={2024}
}

@article{cunningham2023sparse,
  title={Sparse autoencoders find highly interpretable features in language models},
  author={Cunningham, Hoagy and Ewart, Aidan and Riggs, Logan and Huben, Robert and Sharkey, Lee},
  journal={arXiv preprint arXiv:2309.08600},
  year={2023}
}

@inproceedings{huben2023sparse,
  title={Sparse autoencoders find highly interpretable features in language models},
  author={Huben, Robert and Cunningham, Hoagy and Smith, Logan Riggs and Ewart, Aidan and Sharkey, Lee},
  booktitle={The Twelfth International Conference on Learning Representations},
  year={2023}
}

@article{gao2024scaling,
  title={Scaling and evaluating sparse autoencoders},
  author={Gao, Leo and la Tour, Tom Dupr{\'e} and Tillman, Henk and Goh, Gabriel and Troll, Rajan and Radford, Alec and Sutskever, Ilya and Leike, Jan and Wu, Jeffrey},
  journal={arXiv preprint arXiv:2406.04093},
  year={2024}
}

@article{surkov2024unpacking,
  title={Unpacking sdxl turbo: Interpreting text-to-image models with sparse autoencoders},
  author={Surkov, Viacheslav and Wendler, Chris and Terekhov, Mikhail and Deschenaux, Justin and West, Robert and Gulcehre, Caglar},
  journal={arXiv preprint arXiv:2410.22366},
  year={2024}
}

@article{kim2024textit,
  title={Revelio: Interpreting and leveraging semantic information in diffusion models},
  author={Kim, Dahye and Thomas, Xavier and Ghadiyaram, Deepti},
  journal={arXiv preprint arXiv:2411.16725},
  year={2024}
}

@article{kim2025concept,
  title={Concept Steerers: Leveraging K-Sparse Autoencoders for Controllable Generations},
  author={Kim, Dahye and Ghadiyaram, Deepti},
  journal={arXiv preprint arXiv:2501.19066},
  year={2025}
}

@article{cywinski2025saeuron,
  title={SAeUron: Interpretable Concept Unlearning in Diffusion Models with Sparse Autoencoders},
  author={Cywi{\'n}ski, Bartosz and Deja, Kamil},
  journal={arXiv preprint arXiv:2501.18052},
  year={2025}
}

@article{huang2025tide,
  title={TIDE: Temporal-Aware Sparse Autoencoders for Interpretable Diffusion Transformers in Image Generation},
  author={Huang, Victor Shea-Jay and Zhuo, Le and Xin, Yi and Wang, Zhaokai and Gao, Peng and Li, Hongsheng},
  journal={arXiv preprint arXiv:2503.07050},
  year={2025}
}

@article{liu2025predict,
  title={I Predict Therefore I Am: Is Next Token Prediction Enough to Learn Human-Interpretable Concepts from Data?},
  author={Liu, Yuhang and Gong, Dong and Gao, Erdun and Zhang, Zhen and Huang, Biwei and Gong, Mingming and Hengel, Anton van den and Shi, Javen Qinfeng},
  journal={arXiv preprint arXiv:2503.08980},
  year={2025}
}

@article{joshi2025identifiable,
  title={Identifiable Steering via Sparse Autoencoding of Multi-Concept Shifts},
  author={Joshi, Shruti and Dittadi, Andrea and Lachapelle, S{\'e}bastien and Sridhar, Dhanya},
  journal={arXiv preprint arXiv:2502.12179},
  year={2025}
}

@inproceedings{rajendran2024from,
title={From Causal to Concept-Based Representation Learning},
author={Goutham Rajendran and Simon Buchholz and Bryon Aragam and Bernhard Sch{\"o}lkopf and Pradeep Kumar Ravikumar},
booktitle={Causality and Large Models @NeurIPS 2024},
year={2024},
url={https://openreview.net/forum?id=FcVnIBYbkW}
}

@inproceedings{jiang2024on,
title={On the Origins of Linear Representations in Large Language Models},
author={Yibo Jiang and Goutham Rajendran and Pradeep Kumar Ravikumar and Bryon Aragam and Victor Veitch},
booktitle={Forty-first International Conference on Machine Learning},
year={2024},
url={https://openreview.net/forum?id=otuTw4Mghk}
}

@article{marconato2024all,
  title={All or none: Identifiable linear properties of next-token predictors in language modeling},
  author={Marconato, Emanuele and Lachapelle, S{\'e}bastien and Weichwald, Sebastian and Gresele, Luigi},
  journal={arXiv preprint arXiv:2410.23501},
  year={2024}
}

@book{peters2017elements,
  title={Elements of causal inference: foundations and learning algorithms},
  author={Peters, Jonas and Janzing, Dominik and Sch{\"o}lkopf, Bernhard},
  year={2017},
  publisher={The MIT Press}
}

@InCollection{hitchcock1997probabilistic,
	author       =	{Hitchcock, Christopher},
	title        =	{{Probabilistic Causation}},
	booktitle    =	{The {Stanford} Encyclopedia of Philosophy},
	editor       =	{Edward N. Zalta},
	howpublished =	{\url{https://plato.stanford.edu/archives/spr2021/entries/causation-probabilistic/}},
	year         =	{2021},
	edition      =	{{S}pring 2021},
	publisher    =	{Metaphysics Research Lab, Stanford University}
}

@article{lachapelle2024nonparametric,
  title={Nonparametric partial disentanglement via mechanism sparsity: Sparse actions, interventions and sparse temporal dependencies},
  author={Lachapelle, S{\'e}bastien and L{\'o}pez, Pau Rodr{\'\i}guez and Sharma, Yash and Everett, Katie and Priol, R{\'e}mi Le and Lacoste, Alexandre and Lacoste-Julien, Simon},
  journal={arXiv preprint arXiv:2401.04890},
  year={2024}
}

@article{xu2024sparsity,
  title={A sparsity principle for partially observable causal representation learning},
  author={Xu, Danru and Yao, Dingling and Lachapelle, S{\'e}bastien and Taslakian, Perouz and Von K{\"u}gelgen, Julius and Locatello, Francesco and Magliacane, Sara},
  journal={arXiv preprint arXiv:2403.08335},
  year={2024}
}

@inproceedings{lachapelle2022disentanglement,
  title={Disentanglement via mechanism sparsity regularization: A new principle for nonlinear ICA},
  author={Lachapelle, S{\'e}bastien and Rodriguez, Pau and Sharma, Yash and Everett, Katie E and Le Priol, R{\'e}mi and Lacoste, Alexandre and Lacoste-Julien, Simon},
  booktitle={Conference on Causal Learning and Reasoning},
  pages={428--484},
  year={2022},
  organization={PMLR}
}

@article{lachapelle2022synergies,
  title={Synergies between Disentanglement and Sparsity: Generalization and Identifiability in Multi-Task Learning},
  author={Lachapelle, S{\'e}bastien and Deleu, Tristan and Mahajan, Divyat and Mitliagkas, Ioannis and Bengio, Yoshua and Lacoste-Julien, Simon and Bertrand, Quentin},
  journal={arXiv preprint arXiv:2211.14666},
  year={2022}
}

@inproceedings{yao2023multi,
  title={Multi-View Causal Representation Learning with Partial Observability},
  author={Yao, Dingling and Xu, Danru and Lachapelle, Sebastien and Magliacane, Sara and Taslakian, Perouz and Martius, Georg and von K{\"u}gelgen, Julius and Locatello, Francesco},
  booktitle={The Twelfth International Conference on Learning Representations},
  year={2023}
}

@inproceedings{zheng2024detecting,
  title={Detecting and Identifying Selection Structure in Sequential Data},
  author={Zheng, Yujia and Tang, Zeyu and Qiu, Yiwen and Sch{\"o}lkopf, Bernhard and Zhang, Kun},
  booktitle={International Conference on Machine Learning},
  pages={61498--61525},
  year={2024},
  organization={PMLR}
}

@misc{moran2025towards,
      title={Towards Interpretable Deep Generative Models via Causal Representation Learning}, 
      author={Gemma E. Moran and Bryon Aragam},
      year={2025},
      eprint={2504.11609},
      archivePrefix={arXiv},
      primaryClass={stat.ML}
}

@article{harkonen2020ganspace,
  title={Ganspace: Discovering interpretable gan controls},
  author={H{\"a}rk{\"o}nen, Erik and Hertzmann, Aaron and Lehtinen, Jaakko and Paris, Sylvain},
  journal={Advances in neural information processing systems},
  volume={33},
  pages={9841--9850},
  year={2020}
}

@inproceedings{voynov2020unsupervised,
  title={Unsupervised discovery of interpretable directions in the gan latent space},
  author={Voynov, Andrey and Babenko, Artem},
  booktitle={International conference on machine learning},
  pages={9786--9796},
  year={2020},
  organization={PMLR}
}

@article{shen2020interfacegan,
  title={Interfacegan: Interpreting the disentangled face representation learned by gans},
  author={Shen, Yujun and Yang, Ceyuan and Tang, Xiaoou and Zhou, Bolei},
  journal={IEEE transactions on pattern analysis and machine intelligence},
  volume={44},
  number={4},
  pages={2004--2018},
  year={2020},
  publisher={IEEE}
}

@article{shi2025routesae,
  publtype={informal},
  author={Wei Shi and Sihang Li and Tao Liang and Mingyang Wan and Gojun Ma and Xiang Wang and Xiangnan He},
  title={Route Sparse Autoencoder to Interpret Large Language Models},
  year={2025},
  month={March},
  cdate={1740787200000},
  journal={CoRR},
  volume={abs/2503.08200},
  url={https://doi.org/10.48550/arXiv.2503.08200}
}

@inproceedings{
mudide2025efficient,
title={Efficient Dictionary Learning with Switch Sparse Autoencoders},
author={Anish Mudide and Joshua Engels and Eric J Michaud and Max Tegmark and Christian Schroeder de Witt},
booktitle={The Thirteenth International Conference on Learning Representations},
year={2025},
url={https://openreview.net/forum?id=k2ZVAzVeMP}
}

@article{radford2018improving,
  title={Improving language understanding by generative pre-training},
  author={Radford, Alec and Narasimhan, Karthik and Salimans, Tim and Sutskever, Ilya and others},
  journal={OpenAI blog},
  year={2018}
}

@article{radford2019language,
  title={Language models are unsupervised multitask learners},
  author={Radford, Alec and Wu, Jeffrey and Child, Rewon and Luan, David and Amodei, Dario and Sutskever, Ilya and others},
  journal={OpenAI blog},
  volume={1},
  number={8},
  pages={9},
  year={2019}
}

@article{scholkopf2021toward,
  title={Toward causal representation learning},
  author={Sch{\"o}lkopf, Bernhard and Locatello, Francesco and Bauer, Stefan and Ke, Nan Rosemary and Kalchbrenner, Nal and Goyal, Anirudh and Bengio, Yoshua},
  journal={Proceedings of the IEEE},
  volume={109},
  number={5},
  pages={612--634},
  year={2021},
  publisher={IEEE}
}

@inproceedings{spirtes1995causal,
  title={Causal inference in the presence of latent variables and selection bias},
  author={Spirtes, Peter and Meek, Christopher and Richardson, Thomas},
  booktitle={Proceedings of the Eleventh conference on Uncertainty in artificial intelligence},
  pages={499--506},
  year={1995}
}

@article{hernan2004structural,
  title={A structural approach to selection bias},
  author={Hern{\'a}n, Miguel A and Hern{\'a}ndez-D{\'\i}az, Sonia and Robins, James M},
  journal={Epidemiology},
  volume={15},
  number={5},
  pages={615--625},
  year={2004},
  publisher={LWW}
}

@article{zhang2008completeness,
  title={On the completeness of orientation rules for causal discovery in the presence of latent confounders and selection bias},
  author={Zhang, Jiji},
  journal={Artificial Intelligence},
  volume={172},
  number={16-17},
  pages={1873--1896},
  year={2008},
  publisher={Elsevier}
}

@article{bareinboim2022recovering,
  title={Recovering from selection bias in causal and statistical inference},
  author={Bareinboim, Elias and Tian, Jin and Pearl, Judea},
  journal={Probabilistic and causal inference: The works of Judea Pearl},
  pages={433--450},
  year={2022}
}

@inproceedings{forre2020causal,
  title={Causal calculus in the presence of cycles, latent confounders and selection bias},
  author={Forr{\'e}, Patrick and Mooij, Joris M},
  booktitle={Uncertainty in Artificial Intelligence},
  pages={71--80},
  year={2020},
  organization={PMLR}
}

@inproceedings{correa2019identification,
  title={Identification of causal effects in the presence of selection bias},
  author={Correa, Juan D and Tian, Jin and Bareinboim, Elias},
  booktitle={Proceedings of the AAAI Conference on Artificial Intelligence},
  volume={33},
  pages={2744--2751},
  year={2019}
}

@article{chen2024modeling,
  title={Modeling latent selection with structural causal models},
  author={Chen, Leihao and Zoeter, Onno and Mooij, Joris M},
  journal={arXiv preprint arXiv:2401.06925},
  year={2024}
}

@misc{schuhmann2022laion,
  title        = {LAION-COCO: 600M Synthetic Captions from LAION2B-EN},
  author       = {Schuhmann, Christoph and Köpf, Andreas and Coombes, Theo and Vencu, Richard and Beaumont, Romain and Trom, Benjamin},
  howpublished = {LAION.ai blog},
  month        = sep,
  year         = {2022},
  url          = {https://laion.ai/blog/laion-coco/}
}

@article{lu2022dpm,
  title={Dpm-solver++: Fast solver for guided sampling of diffusion probabilistic models},
  author={Lu, Cheng and Zhou, Yuhao and Bao, Fan and Chen, Jianfei and Li, Chongxuan and Zhu, Jun},
  journal={arXiv preprint arXiv:2211.01095},
  year={2022}
}

@inproceedings{rubenstein2017causal,
  title={Causal consistency of structural equation models},
  author={Rubenstein, Paul K and Weichwald, Sebastian and Bongers, Stephan and Mooij, Joris M and Janzing, Dominik and Grosse-Wentrup, Moritz and Sch{\"o}lkopf, Bernhard},
  booktitle={Proceedings of the Conference on Uncertainty in Artificial Intelligence (UAI)},
  year={2017}
}

@inproceedings{beckers2021equivalent,
  title={Equivalent causal models},
  author={Beckers, Sander},
  booktitle={Proceedings of the AAAI Conference on Artificial Intelligence},
  volume={35},
  pages={6202--6209},
  year={2021}
}

@inproceedings{geiger2021causal,
  title={Causal abstractions of neural networks},
  author={Geiger, Atticus and Lu, Hanson and Icard, Thomas and Potts, Christopher},
  booktitle={Advances in Neural Information Processing Systems},
  volume={34},
  pages={9574--9586},
  year={2021}
}

@inproceedings{beckers2019abstracting,
  title={Abstracting causal models},
  author={Beckers, Sander and Halpern, Joseph Y},
  booktitle={Proceedings of the AAAI Conference on Artificial Intelligence},
  volume={33},
  pages={2678--2685},
  year={2019}
}

@inproceedings{geiger2024finding,
  title={Finding alignments between interpretable causal variables and distributed neural representations},
  author={Geiger, Atticus and Wu, Zhengxuan and Potts, Christopher and Icard, Thomas and Goodman, Noah D},
  booktitle={Conference on Causal Learning and Reasoning},
  year={2024}
}

@inproceedings{xi2023indeterminacy,
  title={Indeterminacy in generative models: Characterization and strong identifiability},
  author={Xi, Quanhan and Bloem-Reddy, Benjamin},
  booktitle={International Conference on Artificial Intelligence and Statistics},
  pages={6912--6939},
  year={2023},
  organization={PMLR}
}

@inproceedings{buchholz2024robustness,
  title={Robustness of Nonlinear Representation Learning},
  author={Buchholz, Simon and Sch{\"o}lkopf, Bernhard},
  booktitle={International Conference on Machine Learning},
  pages={4785--4821},
  year={2024},
  organization={PMLR}
}

@inproceedings{ismail2024concept,
title={Concept Bottleneck Generative Models},
author={Aya Abdelsalam Ismail and Julius Adebayo and Hector Corrada Bravo and Stephen Ra and Kyunghyun Cho},
booktitle={The Twelfth International Conference on Learning Representations},
year={2024},
url={https://openreview.net/forum?id=L9U5MJJleF}
}

@inproceedings{kulkarni2025interpretable,
  title={Interpretable Generative Models through Post-hoc Concept Bottlenecks},
  author={Kulkarni, Akshay and Yan, Ge and Sun, Chung-En and Oikarinen, Tuomas and Weng, Tsui-Wei},
  booktitle={Proceedings of the Computer Vision and Pattern Recognition Conference},
  pages={8162--8171},
  year={2025}
}

@article{chen2019looks,
  title={This looks like that: deep learning for interpretable image recognition},
  author={Chen, Chaofan and Li, Oscar and Tao, Daniel and Barnett, Alina and Rudin, Cynthia and Su, Jonathan K},
  journal={Advances in neural information processing systems},
  volume={32},
  year={2019}
}

@inproceedings{Rymarczyk2021ProtoPShare,
author = {Rymarczyk, Dawid and Struski, \L{}ukasz and Tabor, Jacek and Zieli\'{n}ski, Bartosz},
title = {ProtoPShare: Prototypical Parts Sharing for Similarity Discovery in Interpretable Image Classification},
year = {2021},
isbn = {9781450383325},
publisher = {Association for Computing Machinery},
address = {New York, NY, USA},
url = {https://doi.org/10.1145/3447548.3467245},
doi = {10.1145/3447548.3467245},
pages = {1420–1430},
numpages = {11},
location = {Virtual Event, Singapore},
booktitle = {KDD '21}
}

@inproceedings{donnelly2022deformable,
  title={Deformable protopnet: An interpretable image classifier using deformable prototypes},
  author={Donnelly, Jon and Barnett, Alina Jade and Chen, Chaofan},
  booktitle={Proceedings of the IEEE/CVF conference on computer vision and pattern recognition},
  pages={10265--10275},
  year={2022}
}

@inproceedings{xue2024protopformer,
  title={ProtoPFormer: Concentrating on Prototypical Parts in Vision Transformers for Interpretable Image Recognition},
  author={Xue, Mengqi and Huang, Qihan and Zhang, Haofei and Hu, Jingwen and Song, Jie and Song, Mingli and Jin, Canghong},
  booktitle={IJCAI},
  year={2024}
}

@misc{flux2024,
    author={Black Forest Labs},
    title={FLUX},
    year={2024},
    howpublished={\url{https://github.com/black-forest-labs/flux}},
}

@inproceedings{patashnik2023localizing,
  title={Localizing object-level shape variations with text-to-image diffusion models},
  author={Patashnik, Or and Garibi, Daniel and Azuri, Idan and Averbuch-Elor, Hadar and Cohen-Or, Daniel},
  booktitle={Proceedings of the IEEE/CVF international conference on computer vision},
  pages={23051--23061},
  year={2023}
}

@article{tinaz2025emergence,
  title={Emergence and Evolution of Interpretable Concepts in Diffusion Models},
  author={Tinaz, Berk and Fabian, Zalan and Soltanolkotabi, Mahdi},
  journal={arXiv preprint arXiv:2504.15473},
  year={2025}
}

@inproceedings{mahajan2024prompting,
  title={Prompting hard or hardly prompting: Prompt inversion for text-to-image diffusion models},
  author={Mahajan, Shweta and Rahman, Tanzila and Yi, Kwang Moo and Sigal, Leonid},
  booktitle={Proceedings of the IEEE/CVF Conference on Computer Vision and Pattern Recognition},
  pages={6808--6817},
  year={2024}
}
\bibliographystyle{configurations/iclr2026_conference}

\appendix

\clearpage
\onecolumn
\appendix

\begin{center}
    \Large\textbf{Appendix}
\end{center}
\vspace{0.5cm}

\section{Related Work} \label{app:related_work}
\paragraph{Latent variable identification.}
Identifying latent variables is a cornerstone of representation learning. A significant body of work establishes identifiability for single-level latent variable models, often assuming the availability of auxiliary information like domain or class labels \citep{khemakhem2020variational,khemakhem2020icebeem,hyvarinen2016unsupervised,hyvarinen2019nonlinear,zhang2024causal}. Recently, research into language models has explored the linear representation hypothesis, yielding linear-subspace identifiability for latent variables \citep{reizinger2024cross,liu2025predict,marconato2024all,rajendran2024from,jiang2024on}. Another research direction \citep{brady2023provably,lachapelle2023additive,lachapelle2024nonparametric,xu2024sparsity,lachapelle2022synergies,lachapelle2022disentanglement,zheng2022identifiability,joshi2025identifiable} leverages sparsity for identification but overlooks the causal relationships among latent variables.
Distinct from these approaches, our work formulates the concept space using \emph{hierarchical} models that allow for the explicit modeling of intricate, multi-level conceptual interactions.
Our work also connects to the literature on causal abstraction, which studies how a high-level causal model can be faithfully derived from a low-level one \citep{rubenstein2017causal, geiger2021causal,geiger2024finding,beckers2019abstracting,beckers2021equivalent}. A key distinction is our focus on \emph{component-wise identifiability}, which guarantees that the discovered concepts are equivalent to the true latent variables, providing a stronger foundation for interpretability.
Our work is complementary to important research on weak vs. strong~\citep{xi2023indeterminacy} and approximate identifiability~\citep{buchholz2024robustness}. While much of this literature analyzes single-level models, our framework is the first to establish component-wise identifiability (Definition~\ref{def:componentwise}) for hierarchical selection models. This result fits within the weak identifiability category~~\citep{xi2023indeterminacy}, as do most results in this area. Critically, this level of identifiability is motivated by and sufficient for our downstream tasks, aligning with the principle of ``task-identifiability''~\citep{xi2023indeterminacy}. It provides the necessary guarantee for meaningful interpretation and control without requiring the stricter assumptions of strong identifiability. Moreover, the results on approximate identifiability~\citep{buchholz2024robustness} are encouraging, suggesting that robust representations can be learned even if our minimality conditions are only approximately met.

\section{Proofs} \label{app:proofs}

\subsection{Proof for Theorem~\ref{thm:vision_identification}}

\begin{lemma}[Base Case Visual Concept Identification]\label{lemma:base_case_vision_identification}
    Assume the following data-generating process:
    \begin{align}
        \chasub \sim \Pb{ \chasub | \auxvar}, \, \invsub \sim  \Pb{ \invsub} , \,
        \obs := g( \chasub, \invsub ).
    \end{align}
    We have the following conditions.
    \begin{enumerate}[label=\roman*,leftmargin=2em, topsep=0.5pt, partopsep=0pt, itemsep=-0.0em]
        \item \label{asmp:invertibility_single_level} \textbf{Informativeness}: The function $g(\cdot)$ is a diffeomorphism.
        \item \label{asmp:smooth_density_single_level} \textbf{Smooth Density}: The probability density function $ p( \chasubl, \invsubl | \auxvarl ) $ is smooth.
        \item \label{asmp:linear_independence_single_level} \textbf{Sufficient Variability}: At any value $ \chasubl $ of $ \chasub $, there exist $n( \chasub )+1$ distinct values of $ \auxvar $, denoted as $\{\auxvarl^{(n)}\}_{n=0}^{n( \chasub )}$, such that the vectors $\ww(\chasubl, \auxvarl^{n})-\ww(\chasubl, \auxvarl^{0})$ are linearly independent where $\ww( \chasubl, \auxvarl ) = \Big(
            \frac{\partial \log p \left(\chasubl | \auxvarl \right)}{\partial \chacompl_{1} }, \ldots, \frac{\partial \log p \left(\chasubl | \auxvarl \right)}{\partial \chacompl_{ n( \chasubl ) } }.
        \Big)$
    \end{enumerate}
    If a specification $ \bm\theta $ satisfies \ref{asmp:invertibility_single_level},\ref{asmp:smooth_density_single_level}, and \ref{asmp:linear_independence_single_level}, another specification $ \hat{\bm\theta} $ satisfies \ref{asmp:invertibility_single_level},\ref{asmp:smooth_density_single_level}, and they generate matching distribution $ \Pb{\obs} $, then we can verify that $ \chasub $ and $ \hat{\chasub} $ can be identified up to its subspace. 
\end{lemma}
\begin{proof}
    Since we have matched distributions, it follows that:
    \begin{align}
        p( \obsl | \auxvarl ) = \hat{p} ( \obsl | \auxvarl ).
    \end{align}
    As the generating function $ g $ has a smooth inverse (\ref{asmp:invertibility_single_level}), we can derive:
    \begin{align*}
        p( g( \chasubl, \invsubl ) | \auxvarl ) &= p ( \hat{g}( \hat{\chasubl}, \hat{\invsubl} ) | \auxvarl ) \implies \\
        p( \chasubl, \invsubl | \auxvarl ) \abs{ \mJ_{g^{-1}} } &= \hat{p} ( g^{-1} \circ \hat{g}( \hat{\chasubl}, \hat{\invsubl}) | \auxvarl ) \abs{ \mJ_{g^{-1}} }.
    \end{align*}
    Notice that the Jacobian determinant $ \abs{ \mJ_{g^{-1}} } > 0 $ because of $g(\cdot)$'s invertibility and let $ h := g^{-1} \circ \hat{g}: (\hat{\chasubl}, \hat{\invsubl} ) \mapsto ( \chasubl, \invsubl ) $ which is smooth and has a smooth inverse thanks to those properties of $g $ and $ \hat{g} $.
    It follows that
    \begin{align*}
        p( \chasubl, \invsubl | \auxvarl ) &= \hat{p} ( h ( \hat{\chasubl}, \hat{\invsubl}) | \auxvarl ) \implies \\
        p( \chasubl, \invsubl | \auxvarl ) &= \hat{p} ( \hat{\chasubl}, \hat{\invsubl}| \auxvarl ) \abs{ \mJ_{ h^{-1} } }.
    \end{align*}
    The independence relation in the generating process implies that
    \begin{align} \label{eq:log_equality}
        \log p( \chasubl | \auxvarl ) + \sum_{ i \in [ n( \invsubl) ] } \log p(\invcomp_{i} ) & = \log \hat{p} ( \hat{\chasubl} | \auxvarl ) + \sum_{i \in [n(\hat{\invsubl})]} \log \hat{p} ( \hat{\invcomp}_{i} ) + \log \abs{ \mJ_{ h^{-1} } }.
    \end{align}
    For any realization $ \auxvarl^{0} $, we subtract \eqref{eq:log_equality} at any $ \auxvarl \neq \auxvarl^{0} $ with that at $ \auxvarl^{0} $:
    \begin{align} \label{eq:changing_log_equality}
        \log p( \chasubl | \auxvarl ) - \log p( \chasubl | \auxvarl^{0} ) & = \log \hat{p} ( \hat{\chasubl} | \auxvarl ) - \log \hat{p} ( \hat{\chasubl} | \auxvarl^{0} ).
    \end{align}
    Taking derivative w.r.t. $ \hat{\invcompl}_{j} $ for $ j \in [n(\hat{\invsubl})] $ yields:
    \begin{align} \label{eq:single_equation}
        \sum_{i \in [n(\chasubl)]} \frac{ \partial }{ \partial \chacompl_{i} } ( \log p( \chasubl | \auxvarl ) - \log p( \chasubl | \auxvarl^{0} ) ) \cdot \frac{ \partial \chacompl_{i} }{ \partial \hat{\invcompl}_{j} } = 0.
    \end{align}
    The left-hand side zeros out because $ \hat{\chasubl} $ is not a function of $\hat{\invsubl} $.

    Condition~\ref{asmp:linear_independence_single_level} ensures the existence of at least $ n(\chasubl) $ such equations with $ \auxvarl^{1}, \dots, \auxvarl^{n(\chasubl)} $ that are linearly independent, constituting a full-rank linear system.
    Since the choice of $j \in [\invsubl]$ is arbitrary. It follows that   
    \begin{align} \label{eq:zero_subspace}
        \frac{ \partial \chacompl_{i} }{ \partial \hat{\invcompl}_{j} } = 0,  \forall i \in [n(\chasubl)], j \in [n(\invsubl)].
    \end{align}

    Therefore, the Jacobian matrix $ \mJ_{ h }  $ is of the following structure:
    \begin{equation}
    \scalebox{1.25}{$
        \mathbf{J}_{h} =
        \begin{bmatrix}
        \frac{\partial \invsubl }{\partial \hat{\invsubl} } \;  &
        \frac{\partial \invsubl }{\partial \hat{\chasubl}  } \\[.5em]
        \frac{\partial \chasubl}{\partial \hat{\invsubl} } \;  &
        \frac{\partial \chasubl }{\partial \hat{\chasubl} }.
        \end{bmatrix}
    $}
    \end{equation}
    \eqref{eq:zero_subspace} suggests that the block $ \frac{\partial \chasubl}{\partial \hat{\invsubl} } = 0 $.
    Since $ \mJ_{ h } $ is full-rank, we can deduce that $ \frac{\partial \chasubl }{\partial \hat{\chasubl} } $ must have full row-rank and $ n( \chasubl ) \leq n ( \hat{\chasubl} ) $.
    The sparsity constraint in \eqref{eq:sparsity_constraint} further implies that $ n( \chasubl ) = n ( \hat{\chasubl} ) $. That is, we can correctly identify the dimensionality of the changing subspace $ \chasubl $.
    Moreover, since $ \mJ_{h} $ is full-rank and the block $ \frac{\partial \chasubl}{\partial \hat{\invsubl} } $ is zero, we can derive that the corresponding block $ \frac{\partial \hat{\chasubl}}{\partial \invsubl } $ in its inverse matrix $ \mJ_{h^{-1}} $ is also zero.
    Therefore, there exists an invertible map $ \hat{\chasubl} \mapsto \chasubl $, which concludes the proof.
\end{proof}

\begin{lemma}[Determining Intersection Cardinality from Union Cardinalities] \label{lemma:inverse_pie}
Let $\mathcal{A} = \{A_1, A_2, \ldots, A_n\}$ be a finite collection of finite sets. If for any non-empty subset of indices $K \subseteq \{1, 2, \ldots, n\}$, the cardinality of the union $\card{\union_{k \in K} A_k}$ is known, then for any non-empty subset of indices $S \subseteq \{1, 2, \ldots, n\}$, the cardinality of the intersection $\card{\intersection_{s \in S} A_s}$ can be determined.
\end{lemma}

\begin{proof}
We proceed by induction on the size of the set of indices $S$, denoted by $|S|$, for which we want to determine the intersection cardinality.

\textbf{Base Case:} $|S| = 1$.
Let $S = \{i\}$ for some $i \in \{1, 2, \ldots, n\}$. We aim to determine the cardinality $\card{\intersection_{s \in S} A_s} = \card{A_i}$.
The union of a single set $A_i$ is simply $A_i$ itself. That is, $A_i = \union_{k \in \{i\}} A_k$.
By the premise of the theorem, the cardinality $\card{\union_{k \in \{i\}} A_k}$ is known.
Therefore, $\card{A_i}$ is known. The base case holds.

\textbf{Inductive Hypothesis:}
Assume that for some integer $m \ge 1$, the cardinality of any intersection of $j$ sets, $\card{\intersection_{j \in J} A_j}$, can be determined from the known union cardinalities for all non-empty index sets $J$ such that $1 \le |J| \le m$.

\textbf{Inductive Step:}
We want to show that the cardinality of any intersection of $m+1$ sets can be determined. Let $S_{m+1}$ be an arbitrary non-empty subset of indices from $\{1, 2, \ldots, n\}$ such that $|S_{m+1}| = m+1$. Our goal is to determine $\card{\intersection_{s \in S_{m+1}} A_s}$.

Consider the Principle of Inclusion-Exclusion (PIE) applied to the union of the sets whose indices are in $S_{m+1}$:
$$ \card{ \union_{s \in S_{m+1}} A_s } = \sum_{\emptyset \neq K \subseteq S_{m+1}} (-1)^{|K|-1} \card{ \intersection_{k \in K} A_k } $$
This sum runs over all non-empty subsets $K$ of $S_{m+1}$. We can separate the term where $K = S_{m+1}$ (which corresponds to the intersection of all $m+1$ sets) from the other terms in the sum:
$$ \card{ \union_{s \in S_{m+1}} A_s } = \left( \sum_{\emptyset \neq K \subset S_{m+1}} (-1)^{|K|-1} \card{ \intersection_{k \in K} A_k } \right) + (-1)^{|S_{m+1}|-1} \card{ \intersection_{s \in S_{m+1}} A_s } $$
Here, the sum is now over all non-empty \emph{proper} subsets $K$ of $S_{m+1}$.
We can rearrange this equation to solve for the term $\card{ \intersection_{s \in S_{m+1}} A_s }$:
$$ (-1)^{|S_{m+1}|-1} \card{ \intersection_{s \in S_{m+1}} A_s } = \card{ \union_{s \in S_{m+1}} A_s } - \sum_{\emptyset \neq K \subset S_{m+1}} (-1)^{|K|-1} \card{ \intersection_{k \in K} A_k } $$
Multiplying both sides by $(-1)^{|S_{m+1}|-1}$ (noting that $((-1)^{|S_{m+1}|-1})^2 = 1$):
$$ \card{ \intersection_{s \in S_{m+1}} A_s } = (-1)^{|S_{m+1}|-1} \left( \card{ \union_{s \in S_{m+1}} A_s } - \sum_{\emptyset \neq K \subset S_{m+1}} (-1)^{|K|-1} \card{ \intersection_{k \in K} A_k } \right) $$
Let us analyze the terms on the right-hand side of this equation:
\begin{enumerate}
    \item The factor $(-1)^{|S_{m+1}|-1}$ is a known sign, since $|S_{m+1}| = m+1$.
    \item The term $\card{ \union_{s \in S_{m+1}} A_s }$ is the cardinality of a union of $m+1$ sets. Since $S_{m+1}$ is a non-empty subset of indices, this value is known by the premise of the theorem.
    \item Consider the sum $\sum_{\emptyset \neq K \subset S_{m+1}} (-1)^{|K|-1} \card{ \intersection_{k \in K} A_k }$. Each $K$ in this summation is a non-empty proper subset of $S_{m+1}$. Therefore, the size of each such $K$ satisfies $1 \le |K| \le m$.
    By the Inductive Hypothesis, for any such $K$ (i.e., for any intersection of $j$ sets where $1 \le j \le m$), the cardinality $\card{ \intersection_{k \in K} A_k }$ can be determined from the known union cardinalities.
    Consequently, every term in this summation, including its sign factor $(-1)^{|K|-1}$, is determinable.
\end{enumerate}
Since all components on the right-hand side of the equation are known or can be determined based on the theorem's premise and the inductive hypothesis, the value of $\card{ \intersection_{s \in S_{m+1}} A_s }$ can be determined.

In conclusion, by the principle of mathematical induction, for any non-empty subset of indices $S \subseteq \{1, 2, \ldots, n\}$, the cardinality of the intersection $\card{ \intersection_{s \in S} A_s }$ can be determined if the cardinality of any union $\card{ \union_{k \in K} A_k }$ (for any non-empty $K \subseteq \{1, 2, \ldots, n\}$) is known.
\end{proof}

\begin{lemma}[Intersection Block Identification~\citep{Kong2023understanding}] \label{lemma:intersection_block_identification}
    We assume the following data-generating process:
    \begin{align} 
        [\vv_{1}, \vv_{2}] &= g( \cc, \ss_{1}, \ss_{2} ), \\
        \vv_{1} & = g_{1} (\cc, \ss_{1}), \\
        \vv_{2} & = g_{2} (\cc, \ss_{2}),
    \end{align}
    where $ \cc \in \cC \subset \R^{d_{c}} $, $ \ss_{1} \in \cS \subset \R^{d_{s_{1}}} $, and $ \ss_{2} \in \cS_{2} \subset \R^{d_{s_{2}}}$. 
    Both $g_{1}$ and $g_{2}$ are smooth and have non-singular Jacobian matrices almost everywhere, and $g$ is invertible.
    If $ \hat{g}_{1}: \cZ \to \cV_{1} $ and $ \hat{g}_{2}: \cZ \to \cV_{2} $ assume the generating process of the true model $(g_{1}, g_{2})$ and match the joint distribution $p_{\vv_{1}, \vv_{2}}$, 
    then there is a one-to-one mapping between the estimate $\hat{\cc}$ and the ground truth $\cc$ over $ \cC \times \cS \times \cS $, that is, $\cc$ is block-identifiable.
\end{lemma}

\begin{lemma}[One-level Visual Concept Identification] \label{lemma:one_level_vision_identification}
    Assume the process for visual concepts in \eqref{eq:vision_data_generating_process} with $\Lv =1$. 
    If a model specification $\bm\theta_{\mV}$ satisfies Condition~\ref{cond:vision_identification}, and an alternative specification $\hat{\bm\theta}_{\mV}$ satisfies Conditions~\ref{cond:vision_identification}-\ref{asmp:invertibility} and \ref{cond:vision_identification}-\ref{asmp:smooth_density}, along with a sparsity constraint such that for corresponding $\hat{\late}$ and $\late$:
\begin{align} \label{eq:sparsity_constraint_single_level}
    n( \parents{ \hat{\late} } ) \leq n( \parents{ \late } ),
\end{align}
then, if both models $\bm\theta_{\mV}$ and $\hat{\bm\theta}_{\mV}$ generate the same observed data distribution $\Pb{\obs}$, the latent visual concepts $\lat_{1}$ are component-wise identifiable for every level.
\end{lemma}

\begin{proof}
    For notational convenience, we denote $ \lat_{1} $ as $ \sel $ and $ \dis $ as $ \auxvar $ in this proof. 
    This proof consists of two steps.
    In step one, we identify the connectivity between $\auxcomp$ and $\sele$ variables.
    In step two, we further show the identifiability of the blocks resulting from intersecting the parent sets $ \parents{\auxcomp} $ of multiple $ \auxcomp $ variables.

    \paragraph{Step 1: connectivity identification.}
    Since we have access to the joint distribution $ \Pb{\sel, \auxvar} $, we can derive conditional distributions $ \Pb{ \sel | \{\auxcomp_{i}\}_{i\in \cH} } $ for any index subset $\cH \subseteq [n( \auxvar) ]$.
    By Lemma~\ref{lemma:base_case_vision_identification}, we can identify the dimensionality of the set of variables $ \sel $ that are connected to \emph{any} variable in $ \{\auxcomp_{i}\}_{i\in \cH} $ for any $\cH \subseteq [n( \auxvar) ]$.
    Lemma~\ref{lemma:inverse_pie} implies that we can identify the dimensionality of the set of variables $ \sel $ that are connected to \emph{all} variables in $ \{\auxcomp_{i}\}_{i\in \cH} $ for any $\cH \subseteq [n( \auxvar) ]$.  
    This information gives rise to a partition of $\sele$ components, in which each part is connected to the same set of $ \auxcomp $ variables.
    Therefore, we have identified the bipartite graph between $ \sel $ and $ \auxvar $ up to a permutation.
    
    \paragraph{Step 2: intersection block identification.}
    Denote the indices of $ \sele $ variables that are connected to $ \auxcomp_{i} $ as $ \cI(i) \subseteq [ n(\sel) ]$.
    We denote the block of $ \sele $ components connected to \emph{all} variables in $ \{ \auxcomp_{i}\}_{i \in \cH} $ as $ \sel_{ \cap_{ i \in \cH} \cI(i) } $ for any $\cH \subseteq [n( \auxvar) ]$.
    Thanks to Lemma~\ref{lemma:base_case_vision_identification}, we can identify the block $\sel_{\cI(i)}$ connected to the variable $ \auxcomp_{i} $ for any $i \in [n(\auxvar)]$.
    Lemma~\ref{lemma:intersection_block_identification} allows us to identify the intersection of any two blocks $ \sel_{ \cI(i) \cap \cI(j) } $ for $ i \neq j $.
    Therefore, repeated applications of Lemma~\ref{lemma:intersection_block_identification} leads to the identification of the intersection block $\sel_{ \cap_{ i \in \cH} \cI(i) }$ for any $\cH \subseteq [n( \auxvar) ]$.
    This concludes the proof.
\end{proof}

\visionidentificationconditions*
\visionidentification*

\begin{proof}
    By Lemma~\ref{lemma:one_level_vision_identification}, we can identify the set of variables $ \lat_{1} $ that are directly connected to the text variables $ \dis $ and their causal graph.
    Treating the identified $ \lat_{1} $ as the $ \auxvar $ in Lemma~\ref{lemma:one_level_vision_identification}, we can further identify $ \lat_{2} $.
    Repeating this procedure yields the identifiability of the entire model. 
\end{proof}

\section{Key Concept Discussions} \label{app:dicussions}

\paragraph{The roles and purposes of ``Selection-based hierarchy and causality minimality''.}
The selection-based hierarchy and causal minimality are constraints on the natural data distribution (images or text), which is a standard modeling practice in causal representation learning~\citep{scholkopf2021toward}. Specifically, the selection-based hierarchy considers concepts as effects of their constituent parts~\citep{zheng2024detecting}, while causal minimality assumes this underlying causal graph is sparse in a specific way (e.g., Condition~\ref{cond:vision_identification}-\ref{asmp:sparsity}).

\paragraph{``Innate'' hierarchical concept graphs.}
``Innate'' refers to the causal structure inherent in the natural data-generating process itself. Latent concepts in the real world interact (e.g., `eyes' and `nose' are components of a `face'), forming a pre-existing causal structure which we refer to as the "innate concept graph."

\paragraph{True latent variables and their verifications.}
``True latent variables'' follow the standard notion in causal representation learning~\citep{scholkopf2021toward}: they are the disentangled, interpretable, semantic factors of the real-world data-generating process (e.g., age, object pose). This is in contrast to a deep learning model's learned features, which are often an entangled, uninterpretable mixture optimized for a specific training objective. Aligning learned features with true latent variables (referred to as ``identification'') is the central goal, as it enables reliable interpretation (e.g., ``this feature is age'') and precise control (e.g., ``increase this feature to make the face older''). This is a fundamental question that our work addresses through both theoretical guarantees and empirical validation. Our work provides the guarantee that if the data-generating process fulfills the property of causal minimality and our learning objective enforces this (e.g., via sparsity), the model's learned features are provably equivalent to the true latent variables. We then validate this empirically via intervention, a standard practice in causal research~\citep{scholkopf2021toward}. Our experiments (Figure~\ref{fig:vision_causal}) show that manipulating the theoretically identified features provides semantic control over the generated output, providing evidence that these features are the meaningful causal levers of the generative process.

\paragraph{Validity of the conditions.}
While assumptions on the unobserved data-generating process may not be validated directly, we have reasoned for the plausibility of our conditions by reflecting on natural properties of real-world data. Beyond standard regularity assumptions like smoothness and variability~\citep{khemakhem2020variational,khemakhem2020icebeem,hyvarinen2016unsupervised,hyvarinen2019nonlinear,zhang2024causal}, our key minimality condition---Sparse Connectivity (Condition~\ref{cond:vision_identification}-\ref{asmp:sparsity})---is motivated by the observation that concepts typically arise from a sparse set of causes~\citep{lachapelle2024nonparametric,xu2024sparsity,lachapelle2022disentanglement,zheng2022identifiability,moran2021identifiable}. Perhaps a more convincing validation is the empirical results. Our experiments provide strong indicative support for this assumption: by actively enforcing sparsity via SAEs, we successfully extract a meaningful concept hierarchy (Figure~\ref{fig:vision_causal}) that is otherwise dense and not easily interpretable. This success provides support for the usefulness of our overall approach and the validity of our assumptions. We acknowledge that these assumptions, like any in this field, may not hold universally. Fortunately, our strong empirical results suggest they seem effective and plausible for the complex, real-world data we study.

\paragraph{Concept variable interpretation.}
Our theory proves the existence of a clean, one-to-one mapping between a learned feature and a true latent variable. This guarantee is what makes a principled interpretation possible in the first place. The subsequent step—assigning a human-understandable description to this now-identified concept—is intrinsically a task that requires human validation. This is a fundamental aspect of all interpretability research (perhaps modern vision-language models have the potential to automate this process).

\paragraph{Comparison with recent work~\citep{cywinski2025saeuron}.}
On the technique side, \citet{cywinski2025saeuron} feature an elegant concept location technique by utilizing the score function, which could significantly benefit our algorithm. For example, we could employ SAeUron~\citep{cywinski2025saeuron} to confirm whether our features at various timesteps match the concept location it identifies.
Our causal learning algorithm explicitly learns the inter-connectivity among concepts across hierarchical levels. Thus, to modify a part of a high-level concept, we could focus our scope on only the variables connected to this specific high-level concept, which lowers the search complexity. In our experiment example, to implement two changes, ``replacing the rock with tree stump'' and ``adding texture to tree stump'', SAeUron may need to perform two independent searches across all timesteps and node indices. Our method can help reduce the search space to only the low-level nodes connected to ``tree stump''. In addition, pinpointing specific diffusion timesteps to intervene on potentially aids in managing undesirable artifacts. Moreover, our explicit concept graph could also give an interpretable, intuitive characterization of the model’s knowledge.
On the message side, \citet{cywinski2025saeuron} propose a novel score function to select the timestep and node index for accurate concept unlearning.
Our work’s focus is to provide concise and informative theoretical conditions to understand concept learning in deep generative models, with potential applications like concept easing or controllable generation. With this work, we hope the theoretical insights will facilitate the development of refined and dedicated methods in the community.

\paragraph{Comparison with recent work~\citep{kim2024textit}.}
Revelio~\citep{kim2024textit} relies on training a classifier on a specific classification dataset. Revelio trains SAEs and a classifier on a specific dataset (e.g., Caltech-101) to evaluate which features and timesteps are most correlated with class labels.
Our work, in contrast, does not involve class labels. 
Our primary contribution is a hierarchical, causal framework designed to interpret the generative process itself. We apply causal discovery algorithms to discover the causal relationships across different levels of concepts without any class labels. We are able to understand how semantic concepts causally relate to one another across different levels of abstraction to form a coherent output (e.g., how ``ear'' and ``mouth'' features causally contribute to a ``cat face'').
Moreover, \citet{kim2024textit} do not perform interventions or analyze the compositional structure of generation, which are the central themes of our paper.

\section{Implementation Details} \label{app:implementation_details}
We present the diagram of our method in Fig.\ref{fig:method_diagram}.

\paragraph{Annotation of the concepts}
To annotate the concepts discovered by SAEs, we use a two-step process:
1. Identify concept-related features.
   For a target concept (e.g., nudity in the unlearning task), we collect a set of prompts related to that concept, generate the corresponding images, and extract the top-K activated feature indices that are consistently triggered across these samples. These shared indices are treated as related to the target concept.
2. Explore causal relationships.  
   After identifying a node (feature) with first step, we use the inferred causal graph to find its parent and child nodes—features closely related to that concept. We then visualize the node’s attribute map (e.g., distinct regions of the cat in Fig.\ref{fig:topic_figure}) to interpret and confirm the concept’s semantics.

\paragraph{Computing resources.}
We use one L40 GPU for training the SAEs and a standard MacBook Pro with an M1 chip for causal discovery. Training one SAE takes around 8 hours.

\paragraph{Vision experiments.}
For the diffusion sampling process, we utilize the \texttt{sde-dpmsolver++}~\citep{lu2022dpm} sampler, which adds stochasticity between successive steps.
We train the K-sparse autoencoder using a latent dimension of 5120, a batch size of 4096, and the Adam optimizer with a learning rate of 0.0001, setting $K=10$.
We use prompts from the Laion-COCO dataset~\citep{schuhmann2022laion}.

Our causal discovery procedure consists of the following steps:

\begin{enumerate}
    \item \textbf{Identify key features for each SAE:} \\
    For every SAE trained at a specific noise level, we first extract the \textit{top-K} feature indices that show the highest average activation on the 10K LAION-COCO subset dataset.

    \item \textbf{Construct binary feature representations:} \\
    We then enumerate all unique feature indices across samples. For each sample, we create a binary feature vector where a value of 1 indicates that the corresponding feature index appears in the sample’s \textit{top-K} list, and 0 otherwise. This results in a \textit{feature–index matrix} representing the activation pattern of features for each SAE at each noise level.

   \item \textbf{Apply causal discovery:} \\
    Using the constructed matrices, we employ the classical causal discovery PC algorithm to infer the causal structure among the feature indices across noise levels. \textit{[PC]} identifies potential directional dependencies, revealing how certain features may causally influence others.
\end{enumerate}

 For the sparsity ablation study, we control the top-K value used in the SAE. Specifically, we train additional SAEs with K=4 and K=100 at timestep 500. To evaluate the effect of sparsity (Figure~\ref{fig:ablation}), we then perform causal discovery by replacing the SAE features with K=10 with those from the K=4 or K=100 models.
Table~\ref{tab:unlearning} evaluates the following baselines:
SD1.4~\citep{rombach2021highresolution}, ESD~\citep{gandikota2023erasing}, SA~\citep{heng2023selective}, CA~\citep{kumari2023ablating}, MACE~\citep{lu2024mace}, UCE~\citep{gandikota2024unified}, RECE~\citep{gong2024reliable}, SDID~\citep{li2024self}, SLD-MAX~\citep{schramowski2023safe}, SLD-STRONG~\citep{schramowski2023safe}, SLD-MEDIUM~\citep{schramowski2023safe}, SD1.4-NegPrompt~\citep{rombach2021highresolution}, SAFREE~\citep{yoon2024safree}, TRASCE~\citep{jain2024trasce}, and ConceptSteer~\citep{kim2025concept}.

\section{Additional Empirical Results} \label{app:empirical_results}

\begin{figure}
    \centering
    \includegraphics[width=1\linewidth]{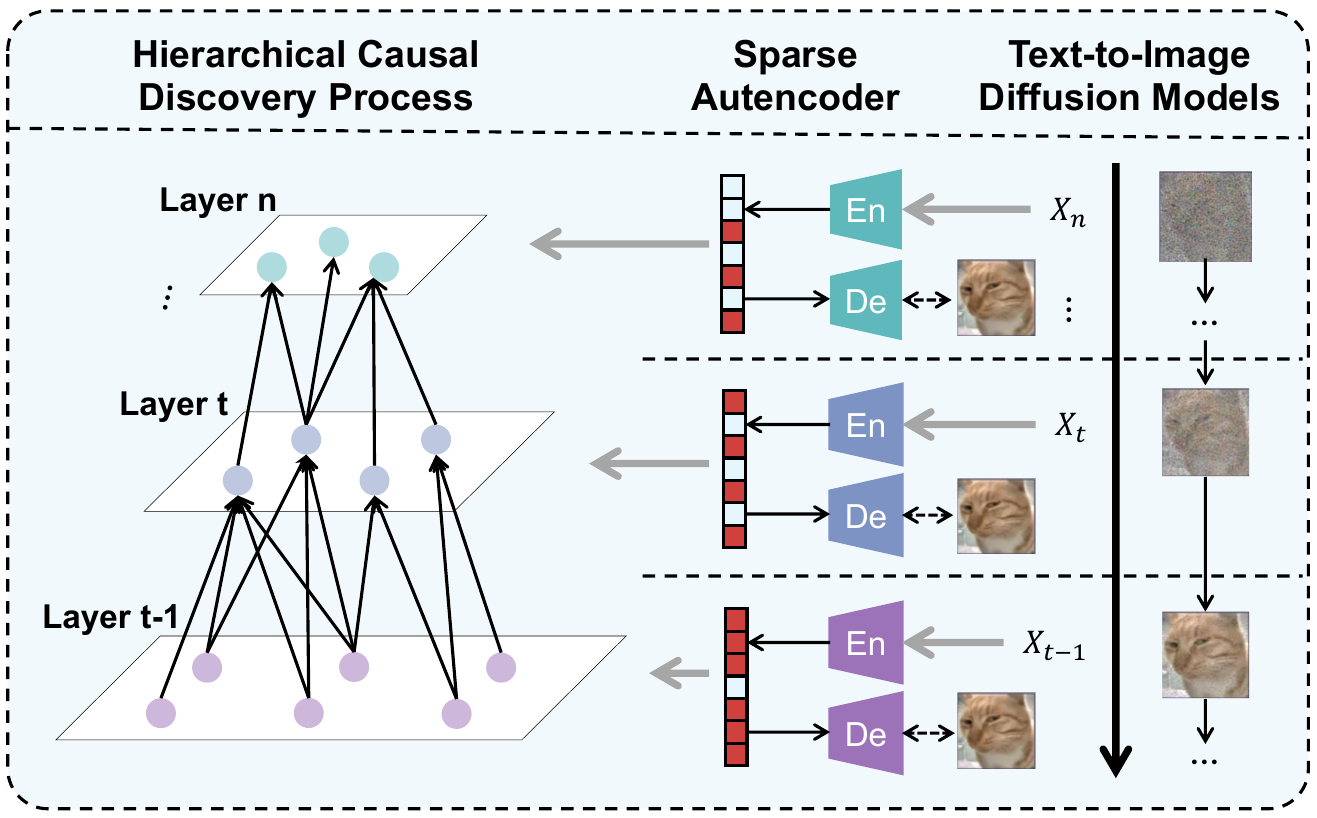}
    \caption{Diagram of our interpretability method. We train SAEs to capture features at different levels and apply causal discovery to construct a hierarchical concept graph.}
    \label{fig:method_diagram}
\end{figure}

\begin{figure} 
    \centering
    \setlength{\tabcolsep}{2pt} 
    \begin{tabular}{cc}
        Input & SD1.4 \\
        \includegraphics[scale=0.1]{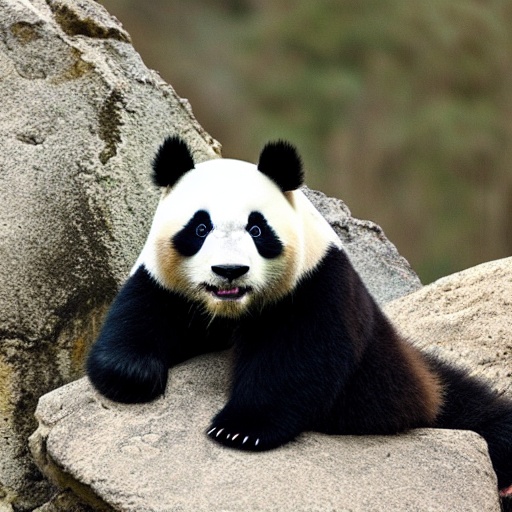} &
        \includegraphics[scale=0.1]{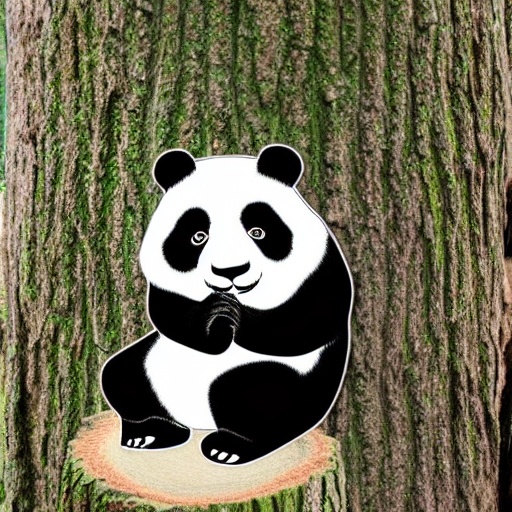} \\
       w/o hier & Ours \\
        \includegraphics[scale=0.1]{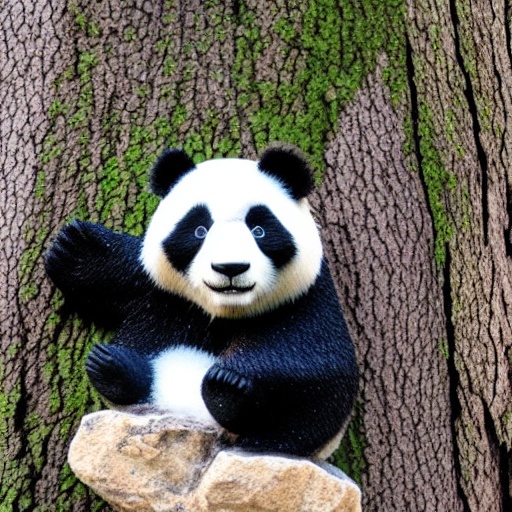} &
        \includegraphics[scale=0.1]{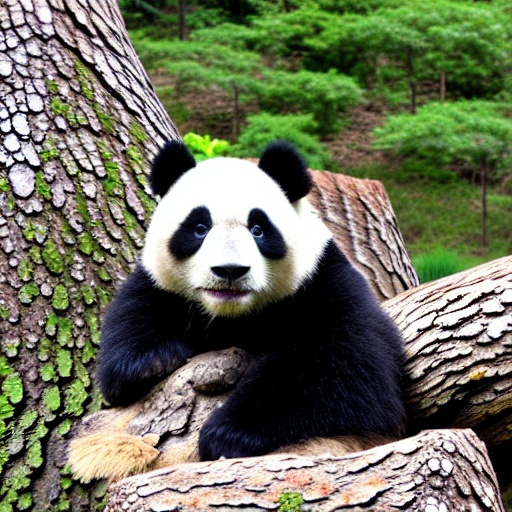} \\
    \end{tabular}
    \vspace{-1em}
    \caption{
        \small
        \textbf{Examples of controllable image generation.}
    }
    \vspace{-1em}
    \label{fig:panda_hier}
\end{figure}

\paragraph{Extension to Flux.1}
Our main experiments are conducted on Stable Diffusion V1.4, which adopts a U-Net architecture. To further validate the generality of our approach, we extend it to Flux.1-Schnell, a 12B text-to-image DiT model. Specifically, we extract features at timesteps 0, 1, 2, and 3 (the model performs inference in only four steps, as it is a distilled model) and train SAEs with the following settings: batch size 4096, learning rate 0.0001, latent dimension 12,288, and top-k = 20. Each SAE is trained on the LAION-COCO dataset for 20,000 steps. We use the last double-stream transformer block (out of 18 double-stream and 38 single-stream blocks) as the feature space (3072 dimensions).
We then perform causal discovery to identify causal dependencies among features. For evaluation, following the setup used for SD1.4, we test our method on the unlearning benchmark datasets (Table \ref{tab:unlearning_appendix}). In particular, we apply negative feature steering at feature index 4390 on timesteps 1, 2, and 3. Our method achieves significantly lower attack success rates on malicious nudity prompts across all benchmarks, demonstrating its robustness and effectiveness on DiT architectures.

\begin{table}[t!]
\centering
\setlength{\tabcolsep}{3.6pt}
\begin{tabular}{lcccccccccc}
\hline
\textbf{Method} & \textbf{I2P $\downarrow$} & \multicolumn{4}{c}{\textbf{RING-A-BELL $\downarrow$}} & \textbf{P4D $\downarrow$} & \textbf{UATK $\downarrow$} & \multicolumn{2}{c}{\textbf{COCO}} \\
\cline{3-6} \cline{9-10}
 &  & K77 & K38 & K16 & AVG &  &  & FID $\downarrow$ & CLIP $\uparrow$ \\
\hline
{\textbf{Flux}}    &  3.08   &50.53&51.58&52.63&51.58&27.15&19.72 & 22.89& 31.57 \\
{\textbf{Ours-Flux}}    &0.94&11.58& 5.26 &4.21& 7.01&3.31 & 4.93 &24.40& 31.54\\ \hline
\end{tabular}
\caption{\textbf{Model unlearning performance on Flux.1-Schnell.} The Flux text-to-image model is susceptible to malicious prompts in benchmark datasets, often producing images containing nudity. By training SAEs on Flux features at different timesteps, we identify the latent representation of the nudity concept and apply negative feature steering to suppress it. This effectively reduces nudity generation while maintaining competitive text-to-image performance on normal prompts from the COCO dataset. }
\label{tab:unlearning_appendix}
\end{table}

\begin{figure}[h]
    \centering
    \setlength{\tabcolsep}{4pt} 
    \begin{tabular}{c@{\hskip 1pt}c@{\hskip 15pt}c@{\hskip 1pt}c@{\hskip 15pt}c@{\hskip 1pt}c}
    \multicolumn{6}{c}{Add fire, + Feature 9678} \\ \hline
        \includegraphics[scale=0.11]{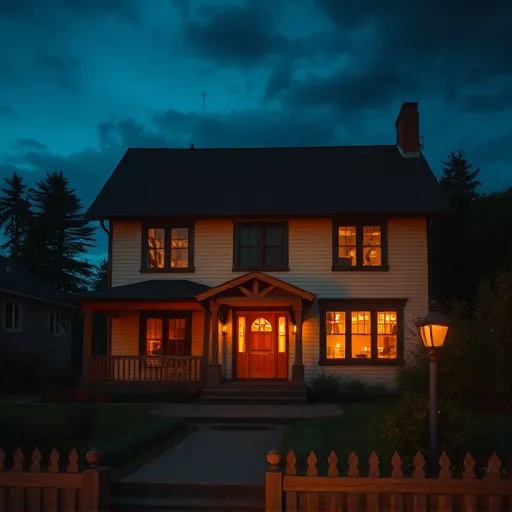} &
        \includegraphics[scale=0.11]{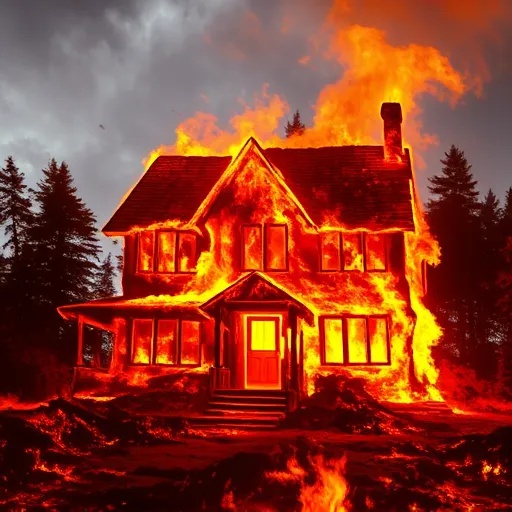} &
        \includegraphics[scale=0.11]{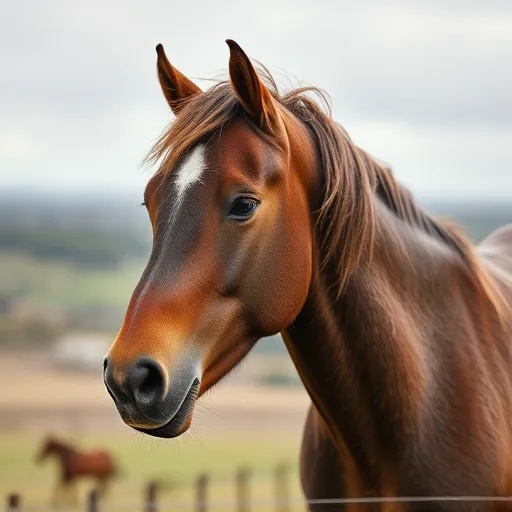} &
        \includegraphics[scale=0.11]{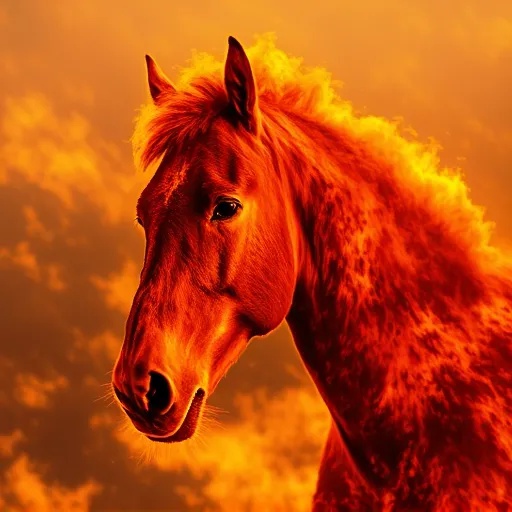} &
        \includegraphics[scale=0.11]{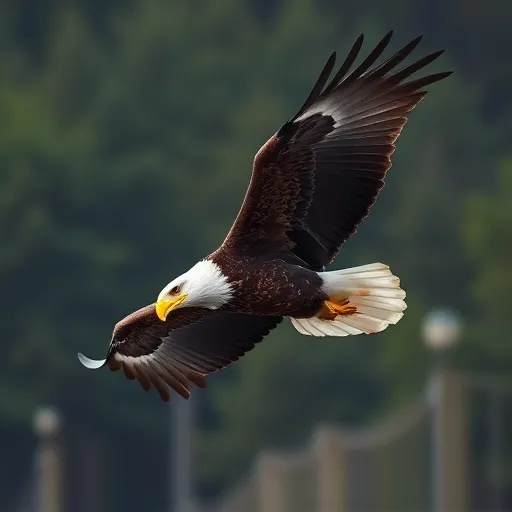} &
        \includegraphics[scale=0.11]{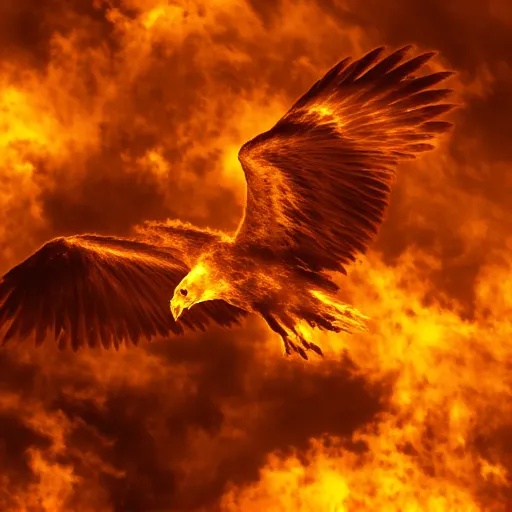} \\
         \multicolumn{6}{c}{Add Mountain, + Feature 4656} \\ \hline
        \includegraphics[scale=0.11]{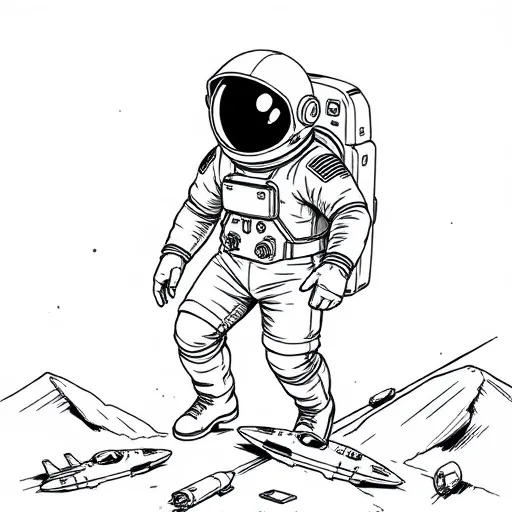} &
        \includegraphics[scale=0.11]{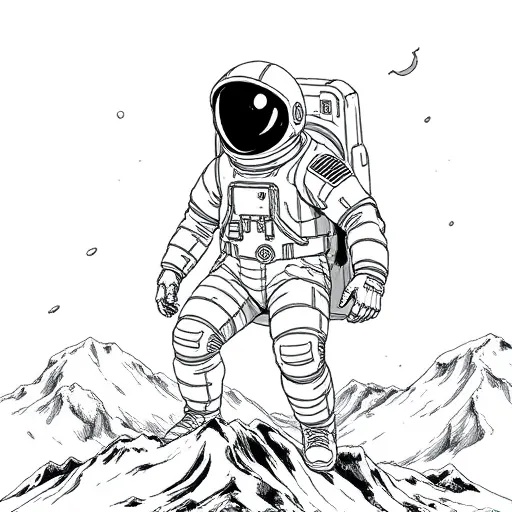} &
        \includegraphics[scale=0.11]{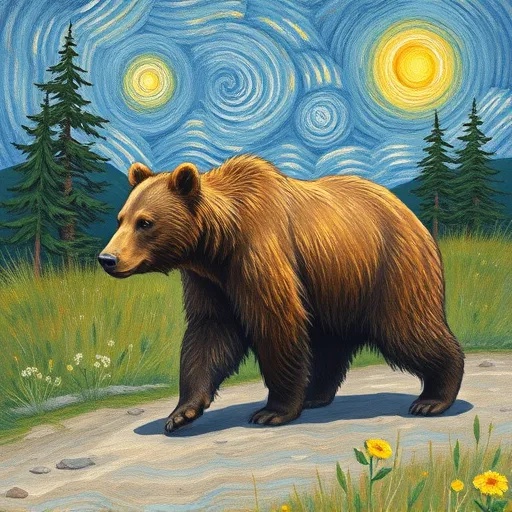} &
        \includegraphics[scale=0.11]{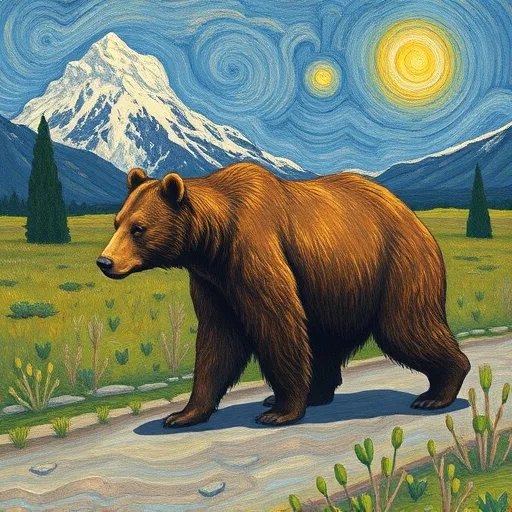} &
        \includegraphics[scale=0.11]{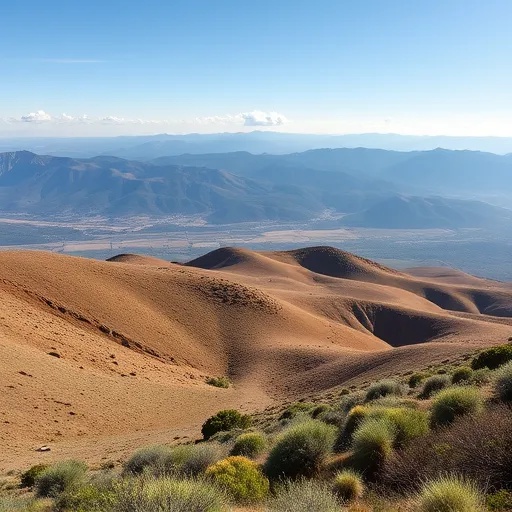} &
        \includegraphics[scale=0.11]{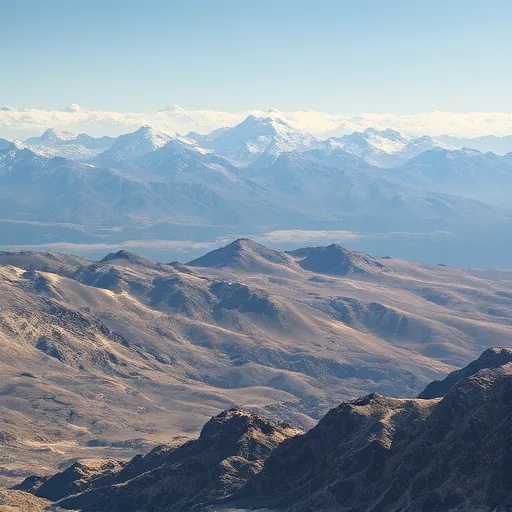} \\
        
    \end{tabular}
    \caption{\textbf{Context-sensitive and objective-agnostic concepts.} Steering the same concept (``fire'' at the top row and ``mountain'' at the bottom row) yields visual changes in the original context, demonstrating that our learned concepts are \emph{context-sensitive} and \emph{shared across objects/classes}.}
    \label{fig:shared_concepts}
\end{figure}

\paragraph{More examples for Figure~\ref{fig:vision_causal}.}
Figure~\ref{fig:vision_causal_more} and Figure~\ref{fig:vision_causal2} contain more examples of Figure~\ref{fig:vision_causal}. For example, node 3641 in the SAE at timestep 899 contains comprehensive information about the panda, as illustrated by the heatmap. When feature steering is applied, it results in the generation of a new panda. Meanwhile, nodes 1026 and 511 in the SAE at timestep 500 represent different components of the panda. At a finer level of detail, nodes 3489, 3880, and 451 in the SAE at timestep 100 capture specific image features. These hierarchical concept graphs effectively illustrate how the panda is generated.

\begin{figure}
    \centering
    \includegraphics[width=1\linewidth]{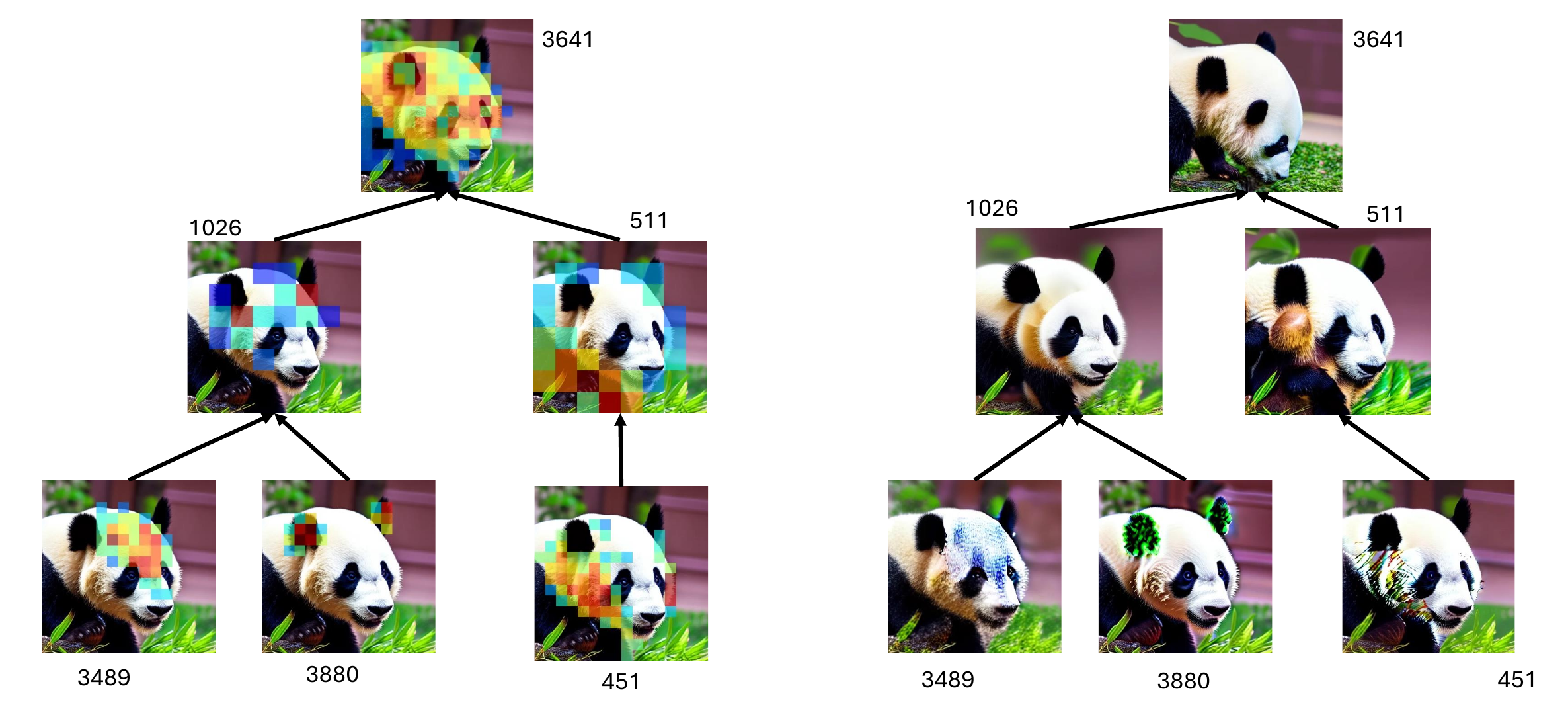}
     \includegraphics[width=1\linewidth]{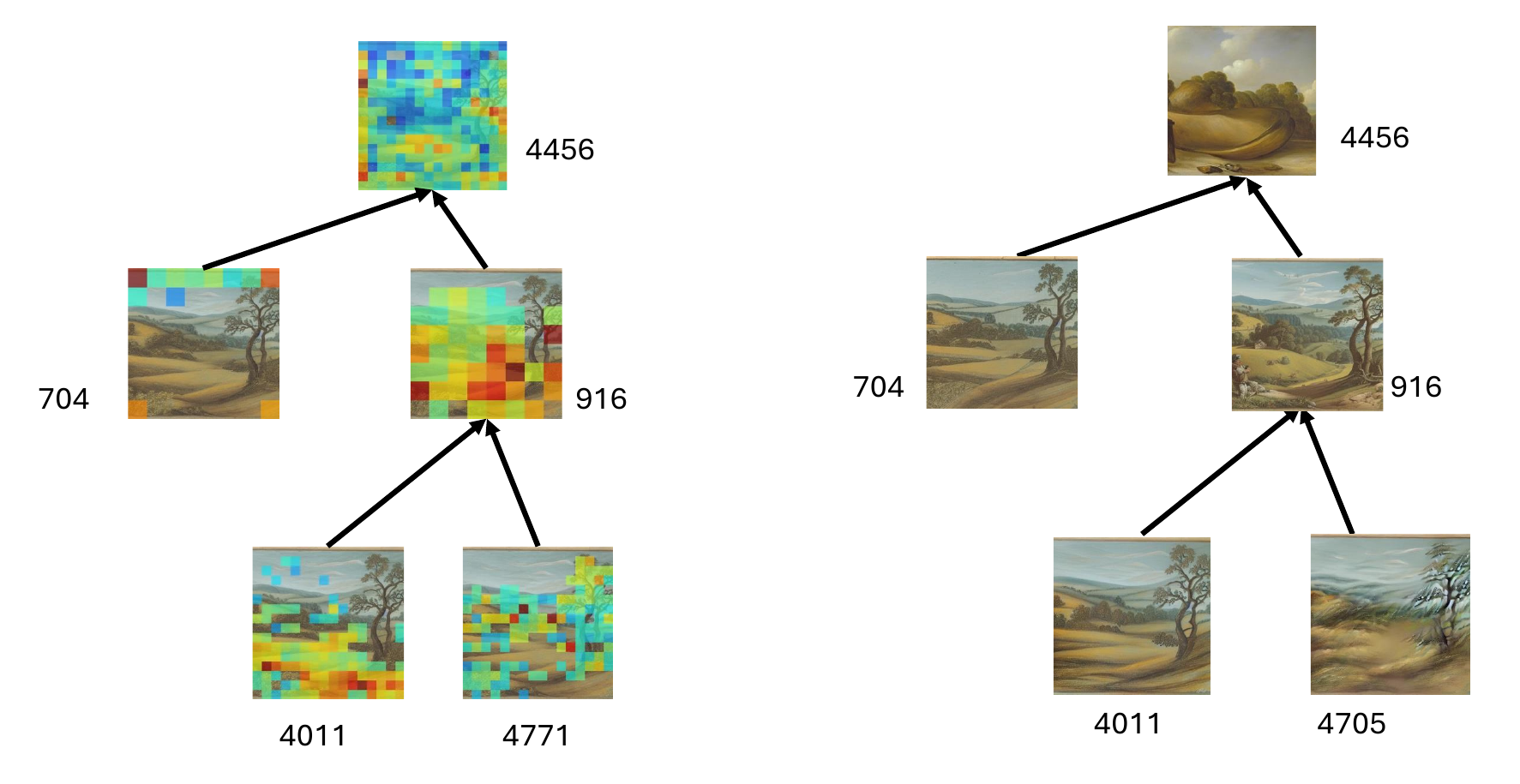}
    \caption{
        \textbf{Discovered hierarchical concept graphs and feature steering visualization for text-to-image generation.}
        We can observe that features on the hierarchical model represent a part-whole relation, and steering a feature yields corresponding visual variation (e.g., the panda's ears).
    }
    \label{fig:vision_causal_more}
\end{figure}

\begin{figure}[t]
\centering
\includegraphics[width=0.9\textwidth]{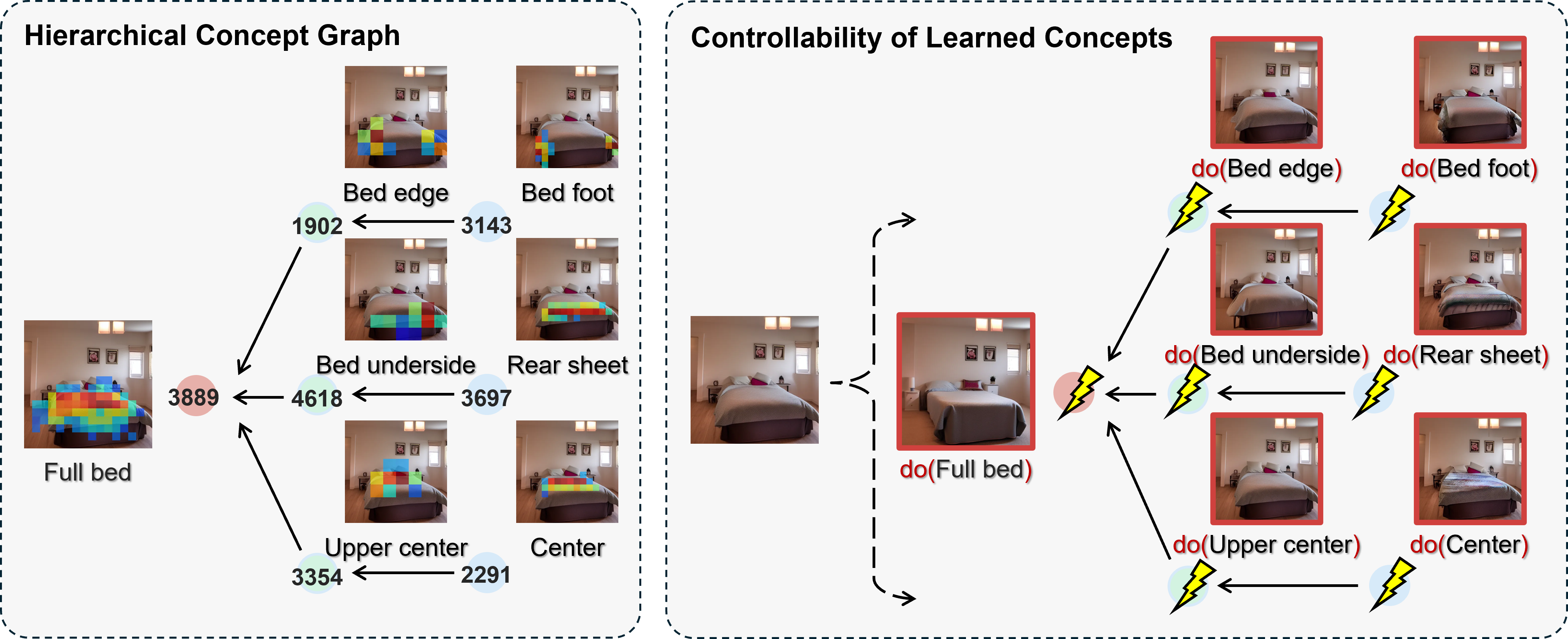}
\caption{
    \small
    \textbf{More examples of the learned hierarchical concept graphs for text-to-image models}. Under appropriate sparsity and noise conditions, our method successfully recovers meaningful hierarchical structures, where each node encodes distinct semantic concepts. On the right, we demonstrate feature steering, where manipulating individual nodes leads to changes in the output that align with their position in the hierarchy -- higher-level nodes produce broader semantic shifts, while lower-level nodes control more fine-grained aspects.
}
\label{fig:vision_causal2}
\end{figure}

\begin{table}
\centering
\begin{tabular}{lcccc}
\toprule
Method & LPIPSe $\uparrow$ & LPIPSu $\downarrow$ & Acce $\downarrow$ & Accu $\uparrow$ \\ 
\midrule
\multicolumn{5}{l}{\textbf{Task: Remove ``Van Gogh''}} \\ \hline
SD-v1.4            & --   & --   & 0.95 & 0.95 \\
CA~\citep{kumari2023ablating}            & 0.30 & 0.13 & 0.65 & 0.90 \\
RECE~\citep{gong2024reliable}         & 0.31 & 0.08 & 0.80 & 0.93 \\
UCE~\citep{gandikota2024unified}          & 0.25 & 0.05 & 0.95 & 0.98 \\
SLD-Medium~\citep{schramowski2023safe}   & 0.21 & 0.10 & 0.95 & 0.91 \\
SAFREE~\citep{yoon2024safree}       & 0.42 & 0.31 & 0.35 & 0.85 \\
Ours               & \textbf{0.53} & 0.26 & \textbf{0.30} & 0.88 \\
\midrule
\multicolumn{5}{l}{\textbf{Task: Remove ``Kelly McKernan''}} \\ \hline
SD-v1.4            & --   & --   & 0.80 & 0.83 \\
CA~\citep{kumari2023ablating}           & 0.22 & 0.17 & 0.50 & 0.76 \\
RECE~\citep{gong2024reliable}         & 0.29 & 0.04 & 0.55 & 0.76 \\
UCE~\citep{gandikota2024unified}          & 0.25 & 0.03 & 0.80 & 0.81 \\
SLD-Medium~\citep{schramowski2023safe}   & 0.22 & 0.18 & 0.50 & 0.79 \\
SAFREE~\citep{yoon2024safree}      & 0.40 & 0.39 & 0.40 & 0.78 \\
Ours               & \textbf{0.48} & 0.20 & \textbf{0.35} & 0.81 \\
\bottomrule \\
\end{tabular}
\caption{\textbf{Results on style removal.} We apply negative feature steering to the node to suppress the styles in the image.
}
\label{tab:style_transfer}
\end{table}

\paragraph{More results for model unlearning}
In addition to the four benchmark datasets in the main paper, 
we report results on another commonly used benchmark dataset with two tasks: \emph{Remove Van Gogh} and \emph{Remove Kelly McKernan} in Table.\ref{tab:style_transfer}. 
We evaluate performance using four metrics: 
LPIPSe (similarity for prompts with the target style), 
LPIPSu (similarity for prompts without the style), 
Acce (how well the target style was removed), 
and Accu (how well other styles were preserved), 
with accuracy ratings assessed using GPT-4o. 
Our method achieves competitive performance across all metrics and tasks.

\paragraph{Understanding the sparsity constraint.}
Figure~\ref{fig:sparsity_ablation} and Table~\ref{tab:sparsity_ablation} contain the ablation study for the sparsity constraint.
We can observe that a proper sparsity strength can indeed give rise to desirable interpretability results, while too small and too large sparsity constraints may be harmful in practice. As shown in Table~\ref{tab:sparsity_ablation}, a low sparsity penalty results in visualized maps with significant overlap. On the other hand, applying a strong sparsity penalty leads to low node coverage, indicating that the nodes alone are insufficient to fully explain the generation of the entire image.

\begin{table}[ht]
    \centering
    \begin{tabular}{c|cc}
    \hline
    & Overlap $\downarrow$ & Coverage $\uparrow$ \\ \hline 
         K=4& 0.108 $\pm$ 0.128 & 26.37$\pm$ 17.24 \\ \hline
         K=10&  0.089 $\pm$ 0.079 & 47.90 $\pm$ 12.50\\ \hline
         K=100 & 0.235 $\pm$ 0.132 & 37.46 $\pm$ 17.31\\ \hline
    \end{tabular}
    \caption{
    \textbf{
        Quantitative ablation results.} 
        We generate 100 panda images using different random seeds and visualize the feature heatmaps at timestep 500. We adjust the top-K value in the SAE at timestep 500 to control the level of sparsity. To evaluate, we compute the intersection-over-union (IoU) of intermediate heatmaps to measure concept disentanglement, and the union of all features to assess coverage. IoU reflects how distinctly the intermediate concepts are represented, while coverage in percentage indicates the extent to which the intermediate nodes collectively account for the image generation. }
    \label{tab:sparsity_ablation}
\end{table}

\begin{figure}
    \centering
    \includegraphics[width=1\linewidth]{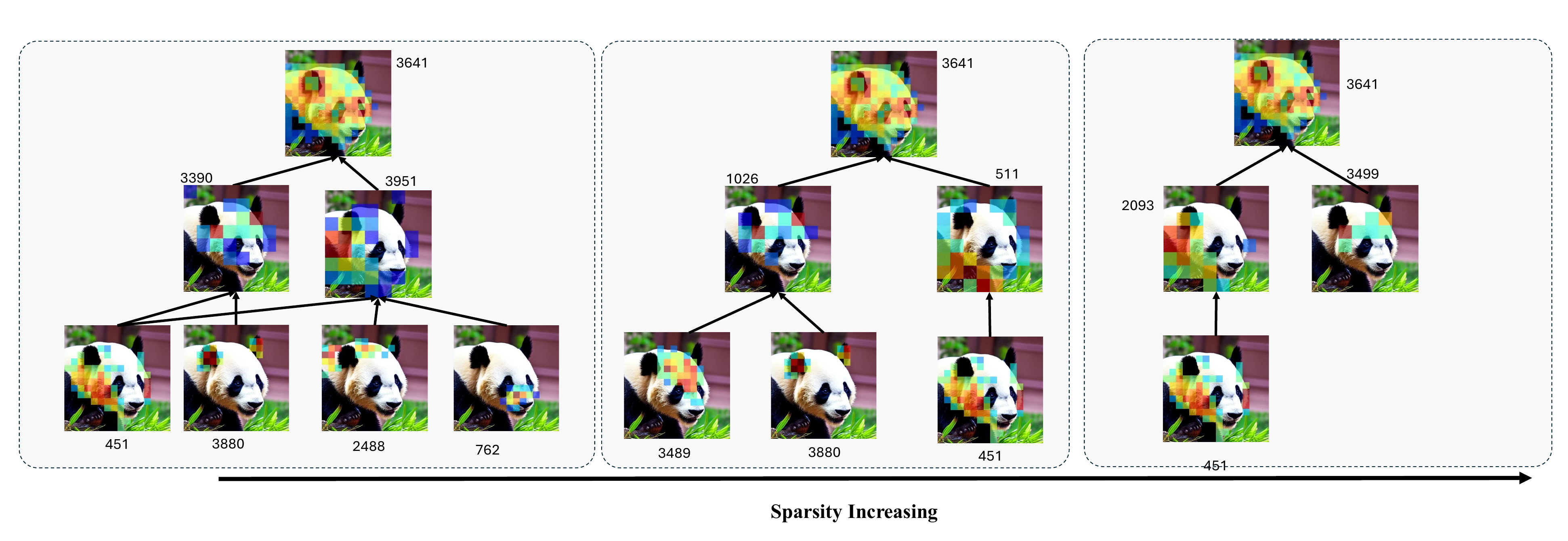}
    \caption{\textbf{Understanding the sparsity constraint.} We adjust the top-K value in the SAE at timestep 500 to control the level of sparsity, effectively modifying the sparsity strength of the SAE at this middle layer. As sparsity decreases, the resulting graph becomes denser, introducing many redundant and semantically irrelevant edges. This reduces the overall interpretability of the concept graph. Conversely, increasing sparsity yields a cleaner, more concise graph. However, if sparsity is too high, it may hinder the formation of a complete and interpretable concept graph necessary for image generation.}
    \label{fig:sparsity_ablation}
\end{figure}

\end{document}